\documentclass{article}

\usepackage{arxiv}
\usepackage[T1]{fontenc}    
\usepackage{hyperref}       
\usepackage{url}            
\usepackage{booktabs}       
\usepackage{amsfonts}       
\usepackage{nicefrac}       
\usepackage{microtype}\usepackage{amsmath}      
\usepackage{xcolor} 
\usepackage{amsthm}
\newtheorem{theorem}{Theorem}
\newtheorem{proposition}{Proposition}
\newtheorem{lemma}{Lemma}
\theoremstyle{definition}
\newtheorem{definition}{Definition}
\usepackage{amssymb}
\usepackage{lipsum}
\usepackage{graphicx}
\graphicspath{ {./images/} }
\usepackage{placeins}
\usepackage{subcaption}

\title{Partial Effective Information Decomposition for synergistic causality}

\author{
 Mingzhe Yang\\
  School of Systems Science\\
  Beijing Normal University\\
   \And
 Shuo Wang \\
  School of Systems Science\\
  Beijing Normal University\\
  \And
 Jiang Zhang \\
  School of Systems Science\\
  Beijing Normal University\\
  \texttt{zhangjiang@bnu.edu.cn} \\
}

\date{}
\renewcommand{\today}{}

\makeatletter
\def\@date{}
\makeatother

\begin{document}
\maketitle
\begin{abstract}
Causality is a central topic in scientific inquiry, yet for complex systems, the identification and analysis of synergistic causation remain a challenging and fundamental problem. In the context of causal relations among multivariate variables, a decomposition framework grounded in interventionist causation is still lacking. To address this gap, this paper proposes Partial Effective Information Decomposition (PEID), a framework that decomposes the influence of multiple source variables on a target variable under maximum-entropy interventions into unique and synergistic information, thereby providing a unified and computable characterization of synergistic causal relations.
Theoretically, in the three-variable case, the proposed framework is compatible with the major axioms of Partial Information Decomposition (PID). Empirically, under maximum-entropy interventions, correlations among input variables are removed, causing redundancy to vanish and thereby enabling PEID to compute synergistic relations.
Furthermore, based on this framework, it is possible to define causal graphs containing hyperedges as well as downward causation, thus offering a unified toolkit for analyzing cross-scale and multivariate causal mechanisms in complex systems. Finally, applying the framework to a machine-learning-based air quality forecasting task on KnowAir-V2, we demonstrate that PEID can extract interpretable inter-station causal structures from a learned dynamical model. These results suggest that PEID provides a general interventionist information-theoretic tool for analyzing multivariate and synergistic causal mechanisms in complex systems.
\end{abstract}

\keywords{Effective information, Partial effective information decomposition, Causal emergence, Synergistic causation, Multiscale causal graph, Downward causation}

\section*{Highlights}

\begin{itemize}

\item Partial effective information decomposition (PEID) decomposes effective information into unique and synergistic causal contributions under maximum-entropy interventions.

\item EI-based causal graphs encode pairwise effective information as directed edges and synergistic effective information as hyperedges.

\item Causal emergence is characterized as multiscale compression of causal structure, in which coarse-graining absorbs microscopic synergistic dependencies into macroscopic variables.

\item PEID extends to continuous nonlinear dynamics and enables interpretable causal analysis of machine-learning-based air quality forecasting models.

\end{itemize}

\section{Introduction}

Understanding causal relations is an important means of characterizing the evolutionary mechanisms of complex systems, revealing the intrinsic dependency structures among variables, and advancing scientific discovery\cite{Granger1969CausalRelations,Shojaie2021GrangerCA,Pearl2009CausalII,MartnezSnchez2024DecomposingCI,Yuan2024EmergenceAC}. Across multiple fields, including statistics, physics, and complex systems science, causal analysis not only provides a theoretical framework for interpreting observational data, but also lays the foundation for predicting system behavior and making intervention decisions.

However, for high-dimensional complex systems, the identification and quantification of causal relations still face significant challenges. One important reason is that once a system contains three or more variables, higher-order causal relations may arise\cite{MartnezSnchez2024DecomposingCI}. In this paper, we use \emph{synergistic causality} to denote a specific form of higher-order causality in which a set of source variables jointly affects a target in a way that cannot be reduced to the causal contribution of any individual source or lower-order subset. XOR-like mechanisms provide the canonical example: neither input is informative about the output by itself, whereas the input pair jointly determines the output. To address this issue, a large body of research based on information decomposition has been devoted to identifying the redundant, unique, and synergistic components in causal relations\cite{Williams2010NonnegativeDO,Mediano2025TowardAU,Bertschinger2013SharedInformation,Harder2013BivariateRedundancy,Finn2018PointwisePID}. However, when the object of study is extended to variables of arbitrary dimension, such methods often encounter combinatorial explosion and rapidly increasing computational complexity. In addition, the notions of causality adopted in these studies are still mostly grounded in the Granger-causality framework\cite{Granger1969CausalRelations,Shojaie2021GrangerCA,MartnezSnchez2024DecomposingCI}, rather than in an intervention-based definition of causation\cite{Pearl2009CausalII}. Granger causality characterizes a directed statistical dependence based on temporal precedence and predictive gain\cite{Granger1969CausalRelations,Shojaie2021GrangerCA}.However, in the absence of additional structural assumptions, this definition is not directly equivalent to causation in the underlying generative mechanism of the system, and it cannot automatically exclude spurious associations induced by latent confounders or common causes\cite{Eichler2013MultipleTimeSeries,Stokes2017ProblemsGC}.

Another important challenge is that complex systems often exhibit significant emergent phenomena: certain stable patterns or causal structures at the macro level often cannot be adequately characterized in a direct and parsimonious manner by local interactions at the micro level\cite{Hoel2013QuantifyingCE,Rosas2020ReconcilingEA,Fromm2005TypesAF}. This implies that causal analysis of complex systems must not only focus on variable relations at a single scale, but also develop multiscale causal models capable of connecting different levels of description. Around this issue, a series of related research directions have emerged in recent years, including causal emergence theory\cite{Hoel2013QuantifyingCE,Hoel2017WhenTM,Rosas2020ReconcilingEA,Zhang2024DynamicalRA}, causal abstraction\cite{Rubenstein2017CausalCO,Zennaro2023QuantifyingCA}, and causal representation learning\cite{Scholkopf2021TowardsCRL}, which provide new theoretical tools for multiscale causal modeling in complex systems from the perspectives of comparing causal efficacy across micro and macro levels, ensuring cross-level model consistency, and discovering high-level causal variables, respectively.

To quantitatively characterize emergent phenomena, Hoel, Albantakis, and Tononi proposed the theory of causal emergence in 2013. Its core idea is to compare the dynamical causal efficacy of a system across different descriptive scales by means of a causal measure grounded in intervention semantics, and thereby determine whether the macro description is causally superior to the micro description\cite{Hoel2013QuantifyingCE,Hoel2017WhenTM}. In fact, as early as the related studies of Integrated Information Theory (IIT), Tononi and colleagues had already introduced the important concept of effective information (EI) to characterize the causal constraint relations among system states\cite{Tononi2003MeasuringII,Tononi2004AnII,Balduzzi2008IntegratedII}. Specifically, EI typically considers applying a maximum-entropy intervention to the state of the system at time $(t)$, and computes the mutual information between the states at times $(t)$ and $(t+1)$ induced by that intervention, thereby measuring the overall causal efficacy manifested by the system dynamics during state evolution. If the EI of the macro dynamics is greater than the EI of the micro dynamics, we say that causal emergence occurs in the system\cite{Hoel2013QuantifyingCE}.

However, the original theory of causal emergence is mainly concerned with comparing the overall causal efficacy of a system across different descriptive scales, that is, determining whether the macro description has a stronger overall causal effect than the micro description. In contrast, in the study of complex systems, we are often more concerned with two further questions: first, what specific causal structure exists among the variables within the system \cite{AlbantakisMarshallHoelTononi2019WhatCausedWhat}, especially whether it contains irreducible higher-order couplings and synergistic causal relations; second, what relationship exists between these within-scale causal structures and causal relations across scales.

A considerable body of research has sought to characterize the causal organization of complex systems at a fixed scale, using causal networks, information-decomposition frameworks, and more recently notions of higher-order causality that represent irreducible multivariate causal relations through hypergraph-like structures \cite{kovrenek2025higher,warrell2020cyclic,Luppi2021WhatItIsLike,VarleyHoel2022Emergence,Varley2024Synergistic,faes2026dissecting,stramaglia2014synergy}. Related studies have further extended this perspective to cross-level phenomena such as downward causation \cite{Mediano2025TowardAU,Rosas2020ReconcilingEA,mediano2022greater}. However, as noted above, most of these methods are still primarily based on empirical distributions \cite{kovrenek2025higher,faes2026dissecting,stramaglia2014synergy}, and only rarely make direct use of intervention-based EI. One important reason is that the classical definition of EI requires imposing maximum-entropy interventions on all possible states of the system, and computing the mutual information between pre- and post-intervention states in conjunction with the system's dynamical transition mechanism; therefore, when the dynamics are unknown, the state space is large, or the variables are continuous, its computation is generally regarded as difficult\cite{Rosas2020ReconcilingEA,Oizumi2016MeasuringII}.

Recent advances in data-driven dynamical modeling have made it increasingly feasible to learn transition mechanisms directly from time-series observations. A broad class of deep time-series models, neural state-space models, and time-series foundation models can now extract predictive dynamical regularities from complex observational trajectories \cite{liang2024foundation,wang2024deep}. They can be interpreted as parameterized predictive distributions. For continuous-valued systems, a common and well-established choice is the conditional Gaussian predictive model, in which the neural network represents both the expected evolution of the next state and the uncertainty around that evolution \cite{nix1994estimating,sluijterman2024optimal}. This provides a natural bridge between machine-learning-based prediction and intervention-based causal analysis: once a learned predictor defines a probabilistic transition mechanism, it can serve as a surrogate dynamics on which maximum-entropy interventions and EI-based analyses can be carried out.

Under this view, NIS+ serves as a concrete example of how this general strategy can be implemented in practice. It learns a dynamical mechanism from time-series data and searches for a coarse-grained macro representation by maximizing macro-level EI \cite{Zhang2022NeuralIS,Yang2025FindingEI}. The resulting macro-dynamics is not only more interpretable, because it captures a lower-dimensional causal structure of the system, but also tends to exhibit stronger out-of-distribution generalization capacity, since the macro-level mechanism is selected according to intervention-based causal effectiveness rather than merely empirical predictive accuracy \cite{Yuan2024EmergenceAC,Yang2025FindingEI}.

Because we require a causal structure grounded in intervention semantics, and because EI can be computed and applied in real systems, this paper proposes a theoretical framework for decomposing EI, namely Partial Effective Information Decomposition (PEID). First, we are able to prove that, in the three-variable case (two source variables and one target variable), PEID is compatible with the axiomatic system satisfied by PID\cite{Yang2025QuantifyingSS}. These axiomatic systems were introduced for the definition of redundant information, and satisfying these axioms means that the computed redundant, unique, and synergy components accord with our physical intuitions\cite{Williams2010NonnegativeDO,Bertschinger2013SharedInformation,Harder2013BivariateRedundancy,Finn2018PointwisePID}: redundant information denotes the information that both source variables can provide about the target variable, unique information denotes the information that a single source variable can provide while the other source variables cannot, and synergy denotes the information that can be provided only when both source variables are present simultaneously. PID provides a complete mathematical axiomatic system in the three-variable case, but in the case of four or more variables, the axiomatic system contains internal contradictions\cite{Lyu2026MultivariatePID}, and it also suffers from combinatorial explosion in computation.

However, the PID framework can be greatly simplified. Our proof reveals that under maximum-entropy interventions, the source variables are mutually independent, and therefore there is no source-side redundancy. Furthermore, we extend this framework to the EI decomposition of an arbitrary number of $n$ source variables with respect to a single target variable, and prove that in discrete systems all information atoms still satisfy basic mathematical properties such as nonnegativity and hierarchical additivity, while the number of atoms that need to be computed is dramatically reduced compared with the original PID.

The structure of the paper is as follows. We first provide the definitions of unique information and synergistic information under PEID, and based on these definitions construct EI-based causal graphs, including hyperedges representing synergistic information. Our primary contribution is not to infer causal relations directly from observational data, but to provide an intervention-based framework for interpreting and analyzing dynamical models from a causal perspective. Furthermore, we define an EI-based measure of downward causation. We also discuss how PEID can be defined and computed for continuous systems. In the experimental part, we first verify in discrete Boolean network systems that the synergistic information in PEID can capture the property that the whole is greater than the sum of its parts. We then present EI causal graphs across multiple scales and analyze downward causation in a toy example. Finally, for atmospheric pollution time-series data with unknown dynamics, we use machine learning to fit the dynamics and apply PEID, thereby identifying the key factors in the synergistic governance of atmospheric pollution.

\section{Definition and Property}
\label{sec:definition}

Consider a discrete-time Markov dynamical system:
\begin{equation}
X_{t+1} \sim p(x_{t+1}\mid x_t), \qquad
X_t = \bigl(X_t^{(1)}, \dots, X_t^{(n)}\bigr).
\label{eq:markov_system}
\end{equation}

Here, $X_t^{(i)}$ denotes the $i$th micro-level random variable. Next, we give the definition of EI.

\subsection{Effective Information}

\label{sec:effective_information}

To the source variables at time $t$, we apply an intervention given by the maximum-entropy distribution, yielding the interventional distribution $q^{\max}(x_t)$. For discrete systems, $q^{\max}$ is simply the uniform distribution over the input state space. The intervention operation can be written as

\begin{equation}
\mathrm{do}\bigl(X_t \sim \mathcal{U}(\Omega_{X_t})\bigr).
\label{eq:do_uniform}
\end{equation}

The post-intervention joint distribution is

\begin{equation}
q(x_t, x_{t+1}) = q^{\max}(x_t) p(x_{t+1}\mid x_t).
\label{eq:joint_q}
\end{equation}

Let $A, B \subseteq {1,\dots,n}$ denote the index subsets on the source side and target side, respectively, with corresponding variables $X_t^A$ and $X_{t+1}^B$. We define effective information (EI) as

\begin{equation}
EI(X_t^A \to X_{t+1}^B)
\equiv
I_{q^{\max}}\bigl(X_t^A; X_{t+1}^B\bigr),
\label{eq:EI_def}
\end{equation}

where $I_{q^{\max}}$ denotes the mutual information computed under the maximum-entropy interventional distribution $q^{\max}$ on the source side.

For a discrete mechanism, EI can be written explicitly as a function of the transition probability matrix (TPM). Let $\Omega_{X_t^A}=\{x_1,\dots,x_M\}$ and $\Omega_{X_{t+1}^B}=\{y_1,\dots,y_L\}$, and define
\begin{equation}
P_{X_t^A \to X_{t+1}^B}:=[p_{ij}]_{M\times L},\qquad
p_{ij}:=P\bigl(X_{t+1}^B=y_j \mid X_t^A=x_i\bigr).
\label{eq:tpm_entry}
\end{equation}
Under the maximum-entropy intervention on the source side, EI admits the explicit form
\begin{equation}
EI\!\left(P_{X_t^A \to X_{t+1}^B}\right)
=
\frac{1}{M}\sum_{i=1}^M\sum_{j=1}^L
p_{ij}\log\frac{M\,p_{ij}}{\sum_{k=1}^M p_{kj}}.
\label{eq:EI_tpm_form}
\end{equation}
This expression makes clear that, EI depends only on the causal mechanism itself.

\subsection{Three-variable case}

We first consider the information decomposition from two source variables to one target variable. Under the maximum-entropy independent intervention on the source side, we have
\begin{equation}
X_t^{(1)} \perp X_t^{(2)}.
\label{eq:ind}
\end{equation}

This independence implies that the redundancy term in PID degenerates to zero. A more detailed proof of this conclusion is given in Appendix~\ref{app:three_variable_proof}.

\begin{proposition}
\label{prop:three_variable}
Let the source variables be $\bigl\{X_t^{(1)}, X_t^{(2)}\bigr\}$ and the target variable be $X_{t+1}$.

Under an intervention that makes the source variables mutually independent, the PID redundancy term is zero, and the unique and synergistic information respectively degenerate to
\begin{equation}
\begin{aligned}
Un\bigl(X_t^{(1)}; X_{t+1} \mid X_t^{(2)}\bigr)
&= EI\bigl(X_t^{(1)} \to X_{t+1}\bigr), \\
Un\bigl(X_t^{(2)}; X_{t+1} \mid X_t^{(1)}\bigr)
&= EI\bigl(X_t^{(2)} \to X_{t+1}\bigr), \\
Syn\bigl(X_t^{(1)}, X_t^{(2)}; X_{t+1}\bigr)
&=
EI\bigl(\bigl(X_t^{(1)}, X_t^{(2)}\bigr) \to X_{t+1}\bigr)
- EI\bigl(X_t^{(1)} \to X_{t+1}\bigr)
- EI\bigl(X_t^{(2)} \to X_{t+1}\bigr).
\end{aligned}
\end{equation}
\end{proposition}

The above conclusion also holds for an arbitrary target random variable $T$.

\subsection{Unique and synergistic information}

\label{sec:unique_synergistic_information}

Let $A\subseteq \{1,\dots,n\}$ be a given subset of source-side indices, and let
$\mathcal{P}=\{M_1,\dots,M_m\}$ be a partition of $X_t^A$, namely
\begin{equation}
M_i\cap M_j=\emptyset \quad (i\neq j),
\qquad
\bigcup_{i=1}^m M_i=X_t^A .
\label{eq:source_partition}
\end{equation}
The definition of the target-side index subset $B$ is completely analogous to that of the source-side index subset in Sec.~\ref{sec:effective_information}, and the corresponding target variable is denoted by $X_{t+1}^B$.

Based on the result in the three-variable case (see Proposition~\ref{prop:three_variable} and Appendix~\ref{app:three_variable_proof}), we extend the same idea to an arbitrary partition $\mathcal{P}$.

\begin{definition}[unique effective information]
\label{def:partition_unique_eid}
For each part $M_i\in\mathcal{P}$, the partition-level unique effective information from $M_i$ to the target $X_{t+1}^B$ is defined as
\begin{equation}
Un_{\mathcal{P}}^{\mathrm{EID}}\bigl(M_i\to X_{t+1}^B\bigr)
:=
EI\bigl(M_i\to X_{t+1}^B\bigr).
\label{eq:partition_unique_eid}
\end{equation}
It quantifies the causal contribution that can be attributed to the part $M_i$ itself under the source-side maximum-entropy intervention.
\end{definition}

\begin{definition}[synergistic causality]
\label{def:partition_synergistic_causality}
Given a partition $\mathcal{P}=\{M_1,\dots,M_m\}$ of $X_t^A$, the strength of synergistic causality from the source set $X_t^A$ to the target $X_{t+1}^B$ relative to $\mathcal{P}$ is defined as
\begin{equation}
Syn_{\mathcal{P}}^{\mathrm{EID}}\bigl(X_t^A\to X_{t+1}^B\bigr)
:=
EI\bigl(X_t^A\to X_{t+1}^B\bigr)
-
\sum_{i=1}^m
EI\bigl(M_i\to X_{t+1}^B\bigr).
\label{eq:partition_synergistic_causality}
\end{equation}
This quantity measures the irreducible causal influence generated jointly by the parts of $\mathcal{P}$, beyond the sum of the causal influences attributable to the parts separately.
\end{definition}

By definition, we immediately obtain the following decomposition identity:
\begin{equation}
EI\bigl(X_t^A\to X_{t+1}^B\bigr)
=
\sum_{i=1}^m
Un_{\mathcal{P}}^{\mathrm{EID}}\bigl(M_i\to X_{t+1}^B\bigr)
+
Syn_{\mathcal{P}}^{\mathrm{EID}}\bigl(X_t^A\to X_{t+1}^B\bigr).
\label{eq:eid_decomposition_identity}
\end{equation}

\begin{definition}[System-level synergy]
\label{def:system_level_eid_synergy}
In particular, when $A=\{1,\dots,n\}$ and the target is taken to be the full next-time state $X_{t+1}$, define the singleton source-side partition as
\begin{equation}
\mathcal{P}_{\mathrm{fine}}
:=
\bigl\{X_t^{(1)},\dots,X_t^{(n)}\bigr\}.
\label{eq:singleton_source_partition}
\end{equation}
The corresponding system-level synergy is defined as
\begin{equation}
\Phi^{\mathrm{EID}}(X_t)
:=
Syn_{\mathcal{P}_{\mathrm{fine}}}^{\mathrm{EID}}
\bigl(X_t\to X_{t+1}\bigr)
=
EI\bigl(X_t\to X_{t+1}\bigr)
-
\sum_{i=1}^n
EI\bigl(X_t^{(i)}\to X_{t+1}\bigr).
\label{eq:system_level_eid_synergy}
\end{equation}
Thus, $\Phi^{\mathrm{EID}}(X_t)$ is the uniquely determined system-level synergistic information induced by the singleton partition, rather than a family of quantities indexed by arbitrary partitions.
\end{definition}

Next, we present a basic property of the source-side synergistic term.

\begin{theorem}[Nonnegativity of source-side synergy]
\label{thm:nonnegativity}
Under the source-side maximum-entropy intervention in the discrete case, for any source index subset $A \subseteq \{1,\dots,n\}$, any target index subset $B \subseteq \{1,\dots,n\}$, and any partition $\mathcal{P}=\{M_1,\dots,M_m\}$ of $X_t^A$, we have
\begin{equation}
0
\le
Syn_{\mathcal{P}}^{\mathrm{EID}}\bigl(X_t^A \to X_{t+1}^B\bigr)
\le
EI\bigl(X_t^A \to X_{t+1}^B\bigr).
\label{eq:synergy_nonnegativity}
\end{equation}
\end{theorem}

\begin{proof}
By definition,
\begin{equation}
Syn_{\mathcal{P}}^{\mathrm{EID}}\bigl(X_t^A \to X_{t+1}^B\bigr)
=
EI\bigl(X_t^A \to X_{t+1}^B\bigr)
-
\sum_{i=1}^m EI\bigl(M_i \to X_{t+1}^B\bigr).
\label{eq:synergy_by_definition}
\end{equation}

Under the factorized intervention on the source side, this synergistic term can be rewritten as the conditional total correlation:
\begin{equation}
Syn_{\mathcal{P}}^{\mathrm{EID}}\bigl(X_t^A \to X_{t+1}^B\bigr)
=
TC_{q^{\max}}\bigl(M_1,\dots,M_m \mid X_{t+1}^B\bigr),
\label{eq:synergy_as_conditional_tc}
\end{equation}
and conditional total correlation is always nonnegative, which yields the left inequality in Eq.~\eqref{eq:synergy_nonnegativity}. The right inequality follows directly from Eq.~\eqref{eq:eid_decomposition_identity}, because each unique term $Un_{\mathcal{P}}^{\mathrm{EID}}\bigl(M_i \to X_{t+1}^B\bigr)=EI\bigl(M_i \to X_{t+1}^B\bigr)$ is itself nonnegative. A more detailed derivation is given in Appendix \ref{app:nonnegativity-proof}.
\end{proof}

\begin{theorem}[Hierarchical additivity]
\label{thm:hierarchicaladditivity}
Let $\mathcal{P}=\{M_1,\dots,M_m\}$ be a partition of $X_t^A$, and let $M_\star \in \mathcal{P}$ denote an arbitrary block of this partition to be further refined. Suppose that $\mathcal{R}=\{R_1,\dots,R_k\}$ is a partition of $M_\star$, that is,
\begin{equation}
R_i \cap R_j=\emptyset \quad (i\neq j),
\qquad
\bigcup_{i=1}^k R_i = M_\star .
\label{eq:refinement_of_mstar}
\end{equation}
Define the refined partition
\begin{equation}
\mathcal{P}'
=
\bigl(\mathcal{P}\setminus\{M_\star\}\bigr)\cup \mathcal{R}.
\label{eq:refined_partition}
\end{equation}
Then
\begin{equation}
Syn_{\mathcal{P}'}^{\mathrm{EID}}\bigl(X_t^A \to X_{t+1}^B\bigr)
=
Syn_{\mathcal{P}}^{\mathrm{EID}}\bigl(X_t^A \to X_{t+1}^B\bigr)
+
Syn_{\mathcal{R}}^{\mathrm{EID}}\bigl(M_\star \to X_{t+1}^B\bigr).
\label{eq:hierarchical_additivity}
\end{equation}
\end{theorem}


This result shows that synergistic information can be hierarchically decomposed along the path from the micro level to the mesoscopic level and then to the macro level, while preserving strict additivity. In other words, the present framework does not require first recovering all finest-grained information atoms by Möbius inversion and then recombining them at the scale of interest; instead, one can directly compute the synergistic term at any partition level of interest. A detailed derivation is given in Appendix \ref{app:hierarchical_proof}.

\begin{proposition}[Maximality of the finest partition]
\label{prop:finest_partition_maximality}
Fix a source subset $A \subseteq \{1,\dots,n\}$ and a target subset $B \subseteq \{1,\dots,n\}$. Define the finest source-side partition associated with $A$ by
\begin{equation}
\mathcal{P}_{\mathrm{fine}}(A)
:=
\{X_t^{(i)}: i\in A\}.
\label{eq:finest_partition_A}
\end{equation}
Then, for any valid partition $\mathcal{P}=\{M_1,\dots,M_m\}$ of $X_t^A$, we have
\begin{equation}
Syn_{\mathcal{P}}^{\mathrm{EID}}\bigl(X_t^A \to X_{t+1}^B\bigr)
\le
Syn_{\mathcal{P}_{\mathrm{fine}}(A)}^{\mathrm{EID}}\bigl(X_t^A \to X_{t+1}^B\bigr).
\label{eq:finest_partition_maximality}
\end{equation}
\end{proposition}


That is, for fixed source subset $A$ and fixed target subset $B$, the synergistic quantity induced by the finest partition attains the maximum among all source-side partitions. In particular, when $A=\{1,\dots,n\}$ and $B=\{1,\dots,n\}$, the right-hand side reduces to $\Phi^{\mathrm{EID}}(X_t)$.

This result follows naturally from Theorems~\ref{thm:nonnegativity} and~\ref{thm:hierarchicaladditivity}. Indeed, according to Theorem~\ref{thm:hierarchicaladditivity}, any valid partition $\mathcal{P}$ can be transformed into the finest partition $\mathcal{P}_{\mathrm{fine}}(A)$ through successive refinements, and each refinement adds a synergistic term associated with the internal structure of the part being further subdivided. By Theorem~\ref{thm:nonnegativity}, all these additional terms are nonnegative under the source-side maximum-entropy intervention. Therefore, the synergistic quantity cannot decrease during the successive refinement process, and hence the finest partition necessarily attains the maximum.

\subsection{Causal graph}
\label{sec:causal_graph}

Based on the above decomposition, we further define an EI causal graph across time. The graph contains two types of causal structures: pairwise EI edges, which represent causal influences attributable to individual source variables, and synergistic hyperedges, which represent irreducible multivariate causal influences generated by source sets.

\subsubsection{Pairwise EI edges}
\label{subsubsec:pairwise_ei_edges}

\begin{definition}[Pairwise EI edge]
\label{def:pairwise_ei_edge}
For any source variable $X_t^{(i)}$ and target variable $X_{t+1}^{(j)}$, we define the existence of a directed pairwise EI edge by
\begin{equation}
i\to j
\quad \Longleftrightarrow \quad
EI\bigl(X_t^{(i)}\to X_{t+1}^{(j)}\bigr)>0 .
\label{eq:ei_edge_existence}
\end{equation}
When such an edge exists, its weight is defined as
\begin{equation}
w_{i\to j}
:=
EI\bigl(X_t^{(i)}\to X_{t+1}^{(j)}\bigr).
\label{eq:ei_edge_weight}
\end{equation}
\end{definition}

According to Eq.~\eqref{eq:EI_tpm_form}, a pairwise EI edge exists only when interventions on the source variable induce a nontrivial change in the distribution of the target variable. Thus, the edge $i\to j$ indicates that intervening on $X_t^{(i)}$ changes the distribution of $X_{t+1}^{(j)}$ under the source-side maximum-entropy intervention, and the edge weight quantifies the strength of this intervention-based mechanistic influence.

\begin{definition}[EI causal graph]
\label{def:ei_causal_graph}
The pairwise EI causal graph is defined as the directed weighted graph
\begin{equation}
\mathcal{G}_{\mathrm{EI}}
=
\bigl(\mathcal{V}_t,\mathcal{V}_{t+1},\mathcal{E}_{\mathrm{EI}},w\bigr),
\label{eq:ei_causal_graph}
\end{equation}
where
\begin{equation}
\mathcal{V}_t
=
\{X_t^{(1)},\dots,X_t^{(n)}\},
\qquad
\mathcal{V}_{t+1}
=
\{X_{t+1}^{(1)},\dots,X_{t+1}^{(n)}\},
\label{eq:ei_causal_graph_vertices}
\end{equation}
and the directed edge set is
\begin{equation}
\mathcal{E}_{\mathrm{EI}}
=
\bigl\{
i\to j:
EI\bigl(X_t^{(i)}\to X_{t+1}^{(j)}\bigr)>0
\bigr\}.
\label{eq:ei_causal_graph_edges}
\end{equation}
Each edge $i\to j\in\mathcal{E}_{\mathrm{EI}}$ is assigned the weight $w_{i\to j}$ defined in Eq.~\eqref{eq:ei_edge_weight}.
\end{definition}

\subsubsection{Synergistic hyperedges}
\label{subsubsec:synergistic_hyperedges}

\begin{definition}[Synergistic hyperedge]
\label{def:synergistic_hyperedge}
Let $A=\{i_1,\dots,i_r\}$ be a source index subset with $r\geq 2$, and let $B=\{j_1,\dots,j_s\}$ be a nonempty target index subset. Define the singleton partition of the source set $X_t^A$ as
\begin{equation}
\mathcal{P}_{\mathrm{fine}}(A)
:=
\bigl\{X_t^{(a)}:a\in A\bigr\}.
\label{eq:hyperedge_singleton_partition}
\end{equation}
A synergistic hyperedge from $A$ to $B$ exists if and only if
\begin{equation}
Syn_{\mathcal{P}_{\mathrm{fine}}(A)}^{\mathrm{EID}}
\bigl(X_t^A\to X_{t+1}^B\bigr)
>0 .
\label{eq:synergistic_hyperedge_existence}
\end{equation}
When such a hyperedge exists, its weight is defined as
\begin{equation}
w_{A\to B}^{\mathrm{syn}}
:=
Syn_{\mathcal{P}_{\mathrm{fine}}(A)}^{\mathrm{EID}}
\bigl(X_t^A\to X_{t+1}^B\bigr),
\label{eq:synergistic_hyperedge_weight_def}
\end{equation}
or equivalently,
\begin{equation}
w_{A\to B}^{\mathrm{syn}}
=
EI\bigl(X_t^A\to X_{t+1}^B\bigr)
-
\sum_{a\in A}
EI\bigl(X_t^{(a)}\to X_{t+1}^B\bigr).
\label{eq:synergistic_hyperedge_weight}
\end{equation}
\end{definition}

The synergistic hyperedge therefore represents a causal influence that is generated jointly by the source variables in $A$ but cannot be reduced to the sum of their individual EI effects on the target. When $|B|=1$, the hyperedge points to a single target variable. When $|B|>1$, it points to a joint target subsystem and describes an irreducible multivariate influence on that subsystem at the next time step.

\begin{definition}[EI causal hypergraph]
\label{def:ei_causal_hypergraph}
The EI causal hypergraph is defined as
\begin{equation}
\mathcal{H}_{\mathrm{EI}}
=
\bigl(\mathcal{V}_t,\mathcal{V}_{t+1},
\mathcal{E}_{\mathrm{EI}},
\mathcal{E}_{\mathrm{syn}},
w,w^{\mathrm{syn}}\bigr),
\label{eq:ei_causal_hypergraph}
\end{equation}
where $\mathcal{E}_{\mathrm{EI}}$ is the pairwise EI edge set defined in Eq.~\eqref{eq:ei_causal_graph_edges}, and
\begin{equation}
\mathcal{E}_{\mathrm{syn}}
=
\Bigl\{
A\to B:
|A|\geq 2,\ 
B\neq \emptyset,\ 
Syn_{\mathcal{P}_{\mathrm{fine}}(A)}^{\mathrm{EID}}
\bigl(X_t^A\to X_{t+1}^B\bigr)>0
\Bigr\}.
\label{eq:ei_causal_hypergraph_hyperedges}
\end{equation}
Each pairwise edge is weighted by $w_{i\to j}$, and each synergistic hyperedge is weighted by $w_{A\to B}^{\mathrm{syn}}$.
\end{definition}

\subsection{Multiscale and Coarse-Graining}
\label{sec:multiscale}

Let $\psi_\ell$ denote the coarse-graining map at scale $\ell$, and define the corresponding macro-level variables by
\begin{equation}
Z_t^{(\ell)}=\psi_\ell(X_t).
\label{eq:coarse_graining_map}
\end{equation}

Let $\mathcal Z_\ell$ denote the state space of $Z_t^{(\ell)}$. To define macro-level effective information, we impose the maximum-entropy intervention directly on the macro variable $Z_t^{(\ell)}$, namely
\begin{equation}
p_{\mathrm{do}}^{(\ell)}(z)=\frac{1}{|\mathcal Z_\ell|},
\qquad z\in\mathcal Z_\ell.
\label{eq:macro_maxent_intervention}
\end{equation}

For each macro-state $z\in\mathcal Z_\ell$, define its preimage cell by
\begin{equation}
\mathcal C_z^{(\ell)}:=\{x\in\mathcal X:\psi_\ell(x)=z\}.
\label{eq:macro_preimage_cell}
\end{equation}

To implement the macro intervention $do\!\left(Z_t^{(\ell)}=z\right)$ at the micro level, we use the maximum-entropy distribution within the corresponding preimage cell:
\begin{equation}
p_\ell\!\left(x\mid do\!\left(Z_t^{(\ell)}=z\right)\right)=
\begin{cases}
\dfrac{1}{|\mathcal C_z^{(\ell)}|}, & x\in \mathcal C_z^{(\ell)},\\[6pt]
0, & x\notin \mathcal C_z^{(\ell)}.
\end{cases}
\label{eq:macro_do_realization}
\end{equation}

This induces the macro-level causal mechanism
\begin{equation}
P_\ell\!\left(z'\mid do(z)\right)
:=
\sum_{x\in\mathcal C_z^{(\ell)}}
p_\ell\!\left(x\mid do\!\left(Z_t^{(\ell)}=z\right)\right)
\sum_{x':\,\psi_\ell(x')=z'}
P\!\left(x'\mid do(x)\right).
\label{eq:macro_tpm}
\end{equation}

In the most common setting, one may apply the same coarse-graining map to both the $t$ and $t+1$ sides, thereby obtaining macro-level dynamics and naturally lifting the EI causal graph to the macro scale:
\begin{equation}
w_{u\to v}^{(\ell)}
:=
EI\!\left(Z_t^{(\ell,u)} \to Z_{t+1}^{(\ell,v)}\right).
\label{eq:macro_pairwise_edge}
\end{equation}

However, the coarse-graining here need not be symmetric on both sides. In practice, one may also coarse-grain only the source side while keeping the target side as the original micro-level whole. In this case, the mechanism induced by the macro intervention is
\begin{equation}
P_\ell\!\left(x_{t+1}\mid do(z)\right)
:=
\sum_{x\in\mathcal C_z^{(\ell)}}
p_\ell\!\left(x\mid do\!\left(Z_t^{(\ell)}=z\right)\right)
P\!\left(x_{t+1}\mid do(x)\right),
\label{eq:one_sided_macro_mechanism}
\end{equation}
and we consider the effective information from the macro variable to the micro-level whole state at the next time step, namely
\begin{equation}
EI\!\left(Z_t^{(\ell)} \to X_{t+1}\right).
\label{eq:one_sided_macro_ei}
\end{equation}

It quantifies the strength of the mechanistic influence exerted by the macro variable on the entire micro-level system at the next time step under maximum-entropy interventions at the macro level.

\subsection{Downward Causation}
\label{sec:downward_causation}
We define downward causation as the joint influence of the whole system on an individual at the next time step.
\begin{definition}[Downward causation]
\label{def:downward_causation}
For a target component $X_{t+1}^{(j)}$ with $j\in\{1,\dots,n\}$, the downward causation strength from the whole system $X_t$ to this component is defined as
\begin{equation}
DC_j
:=
EI\bigl(X_t \to X_{t+1}^{(j)}\bigr)
-
\sum_{i=1}^n
EI\bigl(X_t^{(i)} \to X_{t+1}^{(j)}\bigr).
\label{eq:downward_causation_strength}
\end{equation}
Equivalently, $DC_j$ is the synergistic causality from the singleton partition of the whole source system to the individual target $X_{t+1}^{(j)}$. It measures the part of the causal influence on $X_{t+1}^{(j)}$ that is generated by the joint state of the whole system and cannot be reduced to the sum of pairwise causal influences from individual source variables.
\end{definition}

Let the environment of the target component $X_t^{(j)}$ be $E_t^{(-j)}:=X_t^{\{1,\dots,n\}\setminus\{j\}} .$ Then the downward causation strength can be further decomposed as
\begin{equation}
\begin{aligned}
DC_j
&=
\underbrace{
EI\bigl(X_t\to X_{t+1}^{(j)}\bigr)
-
EI\bigl(X_t^{(j)}\to X_{t+1}^{(j)}\bigr)
-
EI\bigl(E_t^{(-j)}\to X_{t+1}^{(j)}\bigr)
}_{\text{Flexibility}}
\\
&\quad+
\underbrace{
EI\bigl(E_t^{(-j)}\to X_{t+1}^{(j)}\bigr)
-
\sum_{i\neq j}
EI\bigl(X_t^{(i)}\to X_{t+1}^{(j)}\bigr)
}_{\text{synergy within the environment acting on the target individual}} .
\end{aligned}
\label{eq:downward_causation_decomposition}
\end{equation}


Therefore, downward causation can be interpreted as the part of the mechanism composed of the ability of the individual to flexibly respond to environmental changes, namely Flexibility \cite{Yang2025QuantifyingSS}, together with the synergistic component generated by coordinated interactions within the environment and acting on that individual.

\section{Toy Examples}
\label{sec:toy_examples}

We next validate the rationality of the proposed measures on a set of simple systems with known dynamics.

\subsection{Boolean Networks in Which the Whole Is Greater than the Sum of Its Parts}

To verify that the measure $\Phi^{\mathrm{EID}}(X_t)$ can appropriately quantify the property that a system as a whole is greater than the sum of its parts, we follow the benchmark designed by Mediano et al. \cite{Mediano2019MeasuringIntegratedInformation} for evaluating measures of integrated information, and construct six different eight-node Boolean network systems with similar network structures, as shown in Fig.~\ref{fig:exp1net}.

The dynamics use a unified discrete probabilistic Boolean mechanism,
\begin{equation}
P\bigl(x_{t+1}^{(j)}=1 \mid \mathbf{x}_t\bigr)
=
\sigma\!\Bigl(
b_j+\alpha_j\,\mathrm{copy}_j(\mathbf{x}_t)
+\beta_j\,\mathrm{coop}_j(\mathbf{x}_t)
+\gamma_j\,\mathrm{parity}_j(\mathbf{x}_t)
\Bigr),
\label{eq:boolean_general_mechanism}
\end{equation}
where $\sigma(z)=1/(1+e^{-z})$ is the logistic sigmoid. The three terms $\mathrm{copy}_j$, $\mathrm{coop}_j$, and $\mathrm{parity}_j$ represent three types of local mechanisms with different degrees of higher-order dependence. The copy term captures additive pairwise influences, the cooperative term captures an AND-like dependence among multiple inputs, and the parity term captures an XOR-like dependence that is intrinsically synergistic. Specifically, we define
\begin{equation}
\label{eq:three_mechanisms}
\begin{aligned}
\mathrm{copy}_j(\mathbf{x}_t)
&=
\sum_{i \in \mathcal{N}_j^{\mathrm{copy}}}
w_{ij}^{\mathrm{copy}}\left(2x_t^{(i)}-1\right),\\
\mathrm{coop}_j(\mathbf{x}_t)
&=
\prod_{i \in \mathcal{C}_j} x_t^{(i)}
- 2^{-|\mathcal{C}_j|},\\
\mathrm{parity}_j(\mathbf{x}_t)
&=
\eta_j \left(
2 \left[
\left(
\sum_{i \in \mathcal{P}_j} x_t^{(i)}
\right)
\bmod 2
\right]
- 1
\right).
\end{aligned}
\end{equation}
Here, $\mathcal{N}_j^{\mathrm{copy}}$, $\mathcal{C}_j$, and $\mathcal{P}_j$ respectively denote the signed-copy input set, the cooperative input set, and the parity input set acting on node $j$. The coefficient $w_{ij}^{\mathrm{copy}}$ is the weight of the pairwise copy-like influence from node $i$ to node $j$. In the experiments below, we use the unweighted setting $w_{ij}^{\mathrm{copy}}=1$ for every existing copy input $i\in\mathcal{N}_j^{\mathrm{copy}}$; nodes not belonging to $\mathcal{N}_j^{\mathrm{copy}}$ do not contribute to the sum. The default value of $\eta_j$ is $1$.


\begin{figure}[htbp]
\centering
\captionsetup[subfigure]{labelformat=empty}

\begin{subfigure}[b]{0.9\textwidth}
\centering
\includegraphics[width=\textwidth]{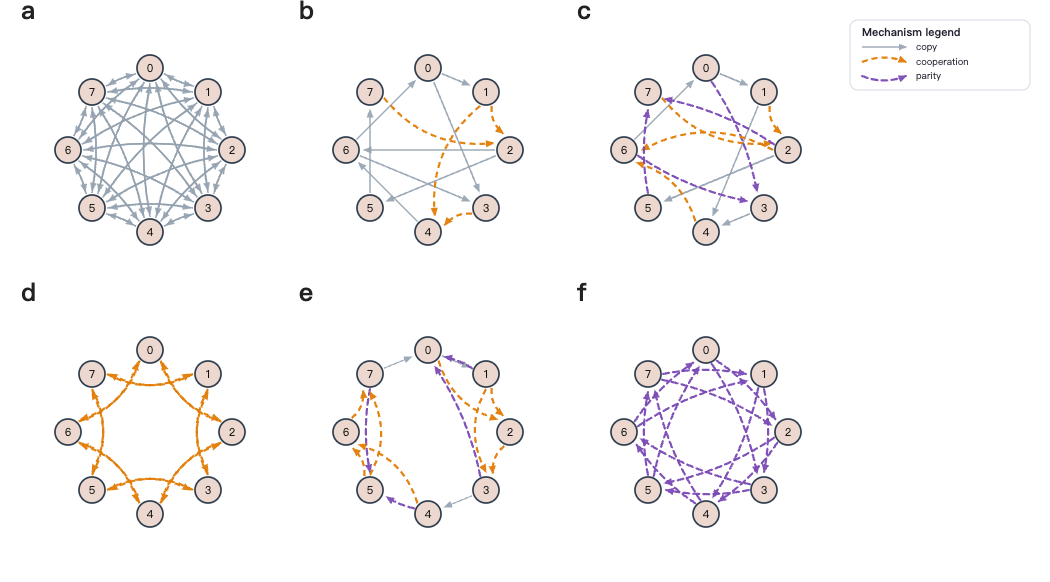}
\caption{Topologies of the six eight-node Boolean networks.}
\label{fig:exp1result_net}
\end{subfigure}

\vspace{0.5em}

\begin{subfigure}[b]{0.48\textwidth}
\centering
\includegraphics[width=\textwidth]{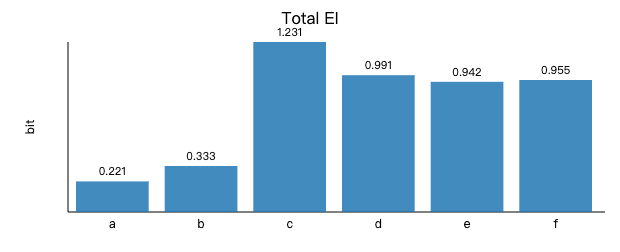}
\caption{}
\label{fig:exp1result_total}
\end{subfigure}
\hfill
\begin{subfigure}[b]{0.48\textwidth}
\centering
\includegraphics[width=\textwidth]{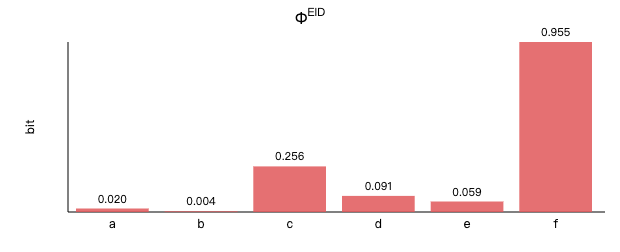}
\caption{}
\label{fig:exp1result_phi}
\end{subfigure}

\caption{Illustration of the six eight-node Boolean networks and their corresponding results.}
\label{fig:exp1net}
\label{fig:exp1result}
\end{figure}

The comparison between the total EI and $\Phi^{\mathrm{EID}}$ for these six Boolean networks is shown in Fig.~\ref{fig:exp1result}. In panel a, the topology is a fully connected network, and all interactions are copy mechanisms. Both EI and $\Phi^{\mathrm{EID}}$ are extremely low, indicating that dense connectivity alone is insufficient; synergistic interaction mechanisms are needed to produce an effect in which the whole is greater than the sum of its parts. In panel b, although the network contains edges associated with the coop mechanism, the connectivity is relatively sparse, and $\Phi^{\mathrm{EID}}$ is the lowest. By comparison, panel c introduces several strongly synergistic parity mechanisms on this basis, leading to a significant increase in both EI and $\Phi^{\mathrm{EID}}$. Panel d is a regular network in which all interactions are coop mechanisms. Its EI is close to that of panel c, but the relative contribution of $\Phi^{\mathrm{EID}}$ is lower. Panel e has a total EI close to that of panel d, but because the network connectivity clearly exhibits community-like modular divisions and is therefore more easily reducible, the relative contribution of $\Phi^{\mathrm{EID}}$ is even lower. The highest $\Phi^{\mathrm{EID}}$ appears in panel f, whose local dynamics are entirely composed of the synergistic interaction mechanism parity, while its network topology does not contain any obvious separable community modules. Therefore, the effect in which the whole is greater than the sum of its parts is the strongest.


This example demonstrates that $\Phi^{\mathrm{EID}}$ can appropriately capture the synergistic mechanism in dynamical systems by which the whole becomes greater than the sum of its parts. A Boolean network system with high $\Phi^{\mathrm{EID}}$ does not need to be particularly densely connected, but it must be able to connect all nodes into an integral whole that is difficult to partition. At the same time, this property also depends on the local dynamical mechanisms, which need to involve synergistic interactions.

\subsection{EI Causal Graph}

We next use a Boolean network with known dynamics to show that the EI-based causal graph defined in Section \ref{sec:causal_graph} can accurately reflect the underlying dynamical mechanism. This is a three-variable discrete Boolean system. Its dynamics are given by
\begin{equation}
x_{0,t+1}=\mathrm{COPY}(x_{2,t}),\qquad
x_{1,t+1}=\mathrm{AND}(x_{0,t},x_{1,t}),\qquad
x_{2,t+1}=\mathrm{AND}(x_{0,t},x_{1,t}).
\label{eq:three-variable-copy-and}
\end{equation}

The original dynamical diagram and the computed EI causal graph are shown in Fig.~\ref{fig:boolean-motif-causal-graphs}. In Fig.~(b), the pairwise edges represent unique effective information, and the edge thickness reflects the magnitude of the value. The dashed lines and circles together form hyperedges, representing synergistic effective information, with the circle size reflecting the magnitude of the value. As shown in the figure, the copy effect from $x_{1,t}$ to $x_{0,t+1}$ corresponds exactly to the unique effective information between them in the causal graph. The AND gate corresponds to a combination of unique information and synergistic information, where the synergistic information is present only in a small amount. In this example, $Syn^{\mathrm{EID}}((x_{1,t},x_{2,t})\to x_{1,t+1})=0.189$. By contrast, the 1 bit in the XOR gate is entirely synergistic effective information. If one only considers pairwise causal edges, then no causal influence from $x_{1,t}$ and $x_{0,t}$ to $x_{2,t+1}$ can be observed. This example clearly shows that an EI causal graph containing hyperedges can faithfully reflect the internal causal structure of a system.

\begin{figure}[htbp]
\centering
\begin{subfigure}[t]{0.5\textwidth}
    \centering
    \includegraphics[width=\textwidth]{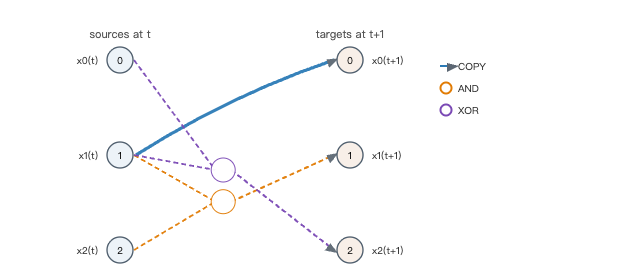}
    \caption{}
    \label{fig:three-variable-mechanism}
\end{subfigure}
\begin{subfigure}[t]{0.49\textwidth}
    \centering
    \includegraphics[width=\textwidth]{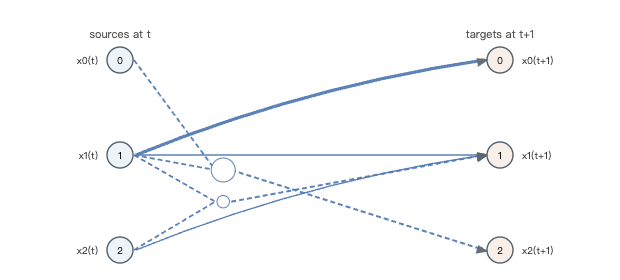}
    \caption{}
    \label{fig:three-variable-ei-graph}
\end{subfigure}
\caption{Comparison between the mechanism diagram and the EI causal graph for the three-variable system. (a)~The true mechanism structure; (b)~the EI-based causal graph.}
\label{fig:boolean-motif-causal-graphs}
\end{figure}




\subsection{Multiscale Causal Graph}

For complex systems, we can perform multiscale modeling, that is, provide causal descriptions of the system dynamics at different scales. Below is a multiscale causal graph for a 6-node discrete Boolean network, whose microscopic nodes are $a_1$, $a_2$, $b_1$, $b_2$, $c_1$, $c_2$, respectively. The dynamics are specified as
\begin{equation}
\begin{aligned}
a_{1,t+1} &= \mathrm{AND}\bigl(\mathrm{OR}(b_{1,t},b_{2,t}),\,\mathrm{OR}(c_{1,t},c_{2,t})\bigr),\\
a_{2,t+1} &= \mathrm{AND}\bigl(\mathrm{OR}(b_{1,t},b_{2,t}),\,\mathrm{XOR}(1,\mathrm{OR}(c_{1,t},c_{2,t}))\bigr),\\
b_{1,t+1} &= \mathrm{AND}\bigl(\mathrm{OR}(c_{1,t},c_{2,t}),\,\mathrm{OR}(a_{1,t},a_{2,t})\bigr),\\
b_{2,t+1} &= \mathrm{AND}\bigl(\mathrm{OR}(c_{1,t},c_{2,t}),\,\mathrm{XOR}(1,\mathrm{OR}(a_{1,t},a_{2,t}))\bigr),\\
c_{1,t+1} &= \mathrm{AND}\bigl(\mathrm{OR}(a_{1,t},a_{2,t}),\,\mathrm{OR}(b_{1,t},b_{2,t})\bigr),\\
c_{2,t+1} &= \mathrm{AND}\bigl(\mathrm{OR}(a_{1,t},a_{2,t}),\,\mathrm{XOR}(1,\mathrm{OR}(b_{1,t},b_{2,t}))\bigr).
\end{aligned}
\label{eq:context-gated-update}
\end{equation}

Fig. \ref{fig:micro-topology-overview} shows the topology of this system. Owing to the special design of the microscopic dynamical structure, we know that the ideal coarse-graining strategy is
\begin{equation}
\begin{aligned}
B_t &:= \mathrm{XOR}(b_{1,t},b_{2,t}) \oplus \mathrm{AND}(b_{1,t},b_{2,t}),\\
C_t &:= \mathrm{XOR}(c_{1,t},c_{2,t}) \oplus \mathrm{AND}(c_{1,t},c_{2,t}),\\
A_t &:= \mathrm{XOR}(a_{1,t},a_{2,t}) \oplus \mathrm{AND}(a_{1,t},a_{2,t}).
\end{aligned}
\label{eq:macro-variables-handcrafted}
\end{equation}

After coarse-graining, the macroscopic dynamics reduce to a three-node cycle: $B \to A$, $C \to B$, and $A \to C$, meaning that there is a one-to-one correspondence between the states at two consecutive time steps and that no macroscopic hyperedges exist. The gated dependencies in the microscopic dynamics that appear to be synergistic mechanisms are absorbed into the internal encoding of the macroscopic nodes. The dynamics shown in the figure are consistent with Eq.~\eqref{eq:context-gated-update}.

\begin{figure}[htbp]
\centering
\includegraphics[width=0.72\linewidth]{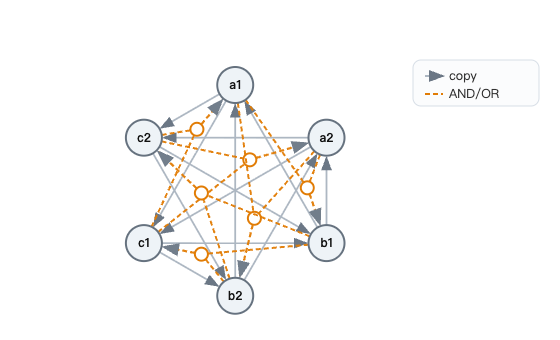}
\caption{Microscopic topology of the handcrafted system. Gray solid lines denote \texttt{copy} relations, while orange circles and dashed lines denote local gating in the form of \texttt{AND/OR}.}
\label{fig:micro-topology-overview}
\end{figure}

To investigate the relationship between synergistic causation and emergent causation, Fig.~\ref{fig:micro-macro-coarse-graining-comparison} presents the multiscale causal graph of this system under optimal coarse-graining: the upper part is the macroscopic EI causal graph, the lower part is the microscopic EI causal graph, and the colored curves on the left and right sides respectively indicate the coarse-graining mappings at time $t$ and time $t+1$. In this case, the three pairs of nodes $a_1, a_2$, $b_1, b_2$, and $c_1, c_2$ are compressed into the macroscopic nodes $A,B,C$, respectively.

\begin{figure}[htbp]
\centering
\includegraphics[width=0.90\linewidth]{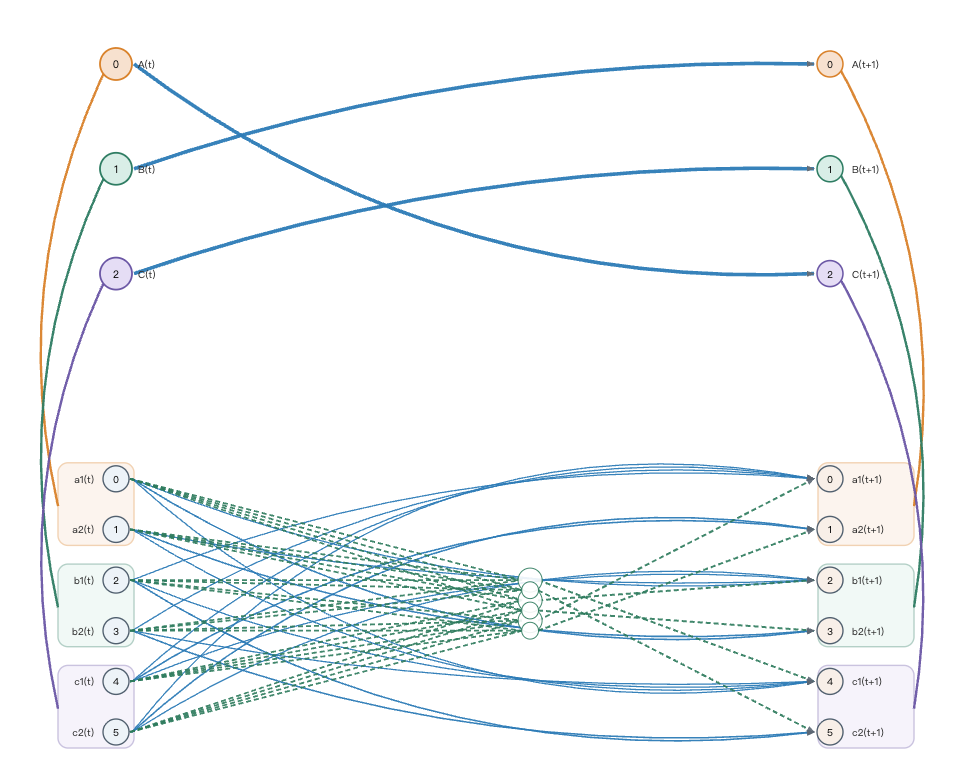}
\caption{Multiscale causal graph under optimal coarse-graining: the upper part is the macroscopic EI causal graph, the lower part is the microscopic EI causal graph, and the colored curves on the left and right sides respectively indicate the coarse-graining mappings at time $t$ and time $t+1$.}
\label{fig:micro-macro-coarse-graining-comparison}
\end{figure}

As can be seen directly from the figure, compared with the intricate causal relations in the microscopic dynamics, EI-optimal coarse-graining yields macroscopic dynamics with a simpler causal structure, showing that this system is essentially characterized by transitions between modules. The total EI of this macroscopic dynamics reaches 3 bit, making it the macroscopic dynamics with maximal EI. The absence of hyperedges in the macroscopic causal graph also indicates that the originally microscopic synergistic causal effects tend to be absorbed into the interior of the macroscopic nodes during coarse-graining. Comparative experimental results for cases without EI maximization are provided in supplementary material section \ref{sec:exp3app}.

\subsection{Downward Causation}

Rosas et al. \cite{Rosas2020ReconcilingEA} previously provided two toy examples to explain dynamical decoupling and downward causation, whose designs are shown in Fig.~\ref{fig:downward-causation-triptych}(a) and (b). Assume that each variable is binary. In the causal decoupling example (Fig.~\ref{fig:downward-causation-triptych}(a)), the macroscopic variable is always obtained from the three microscopic variables through an XOR gate and remains unchanged, whereas the microscopic variables are completely unpredictable from one another. In Fig.~\ref{fig:downward-causation-triptych}(a), $X^{(1)}_{t+1}$ is determined by the XOR of the three variables, while the other two variables change randomly. According to Eq.~\ref{eq:downward_causation_strength}, the downward causation exerted by the causal mechanism on $X^{(1)}_{t+1}$ is indeed 1 bit in this case. It is worth noting that Rosas et al. \cite{Rosas2020ReconcilingEA} did not distinguish the specific target variable when defining downward causation, whereas downward causation should evidently be defined as the influence of the whole system on a particular microscopic unit. In Fig.~\ref{fig:downward-causation-triptych}(b), there is no downward causation for $X^{(2)}_{t+1}$ or $X^{(3)}_{t+1}$.

\begin{figure}[htbp]
\centering
\begin{subfigure}[b]{0.31\textwidth}
    \centering
    \includegraphics[width=\linewidth]{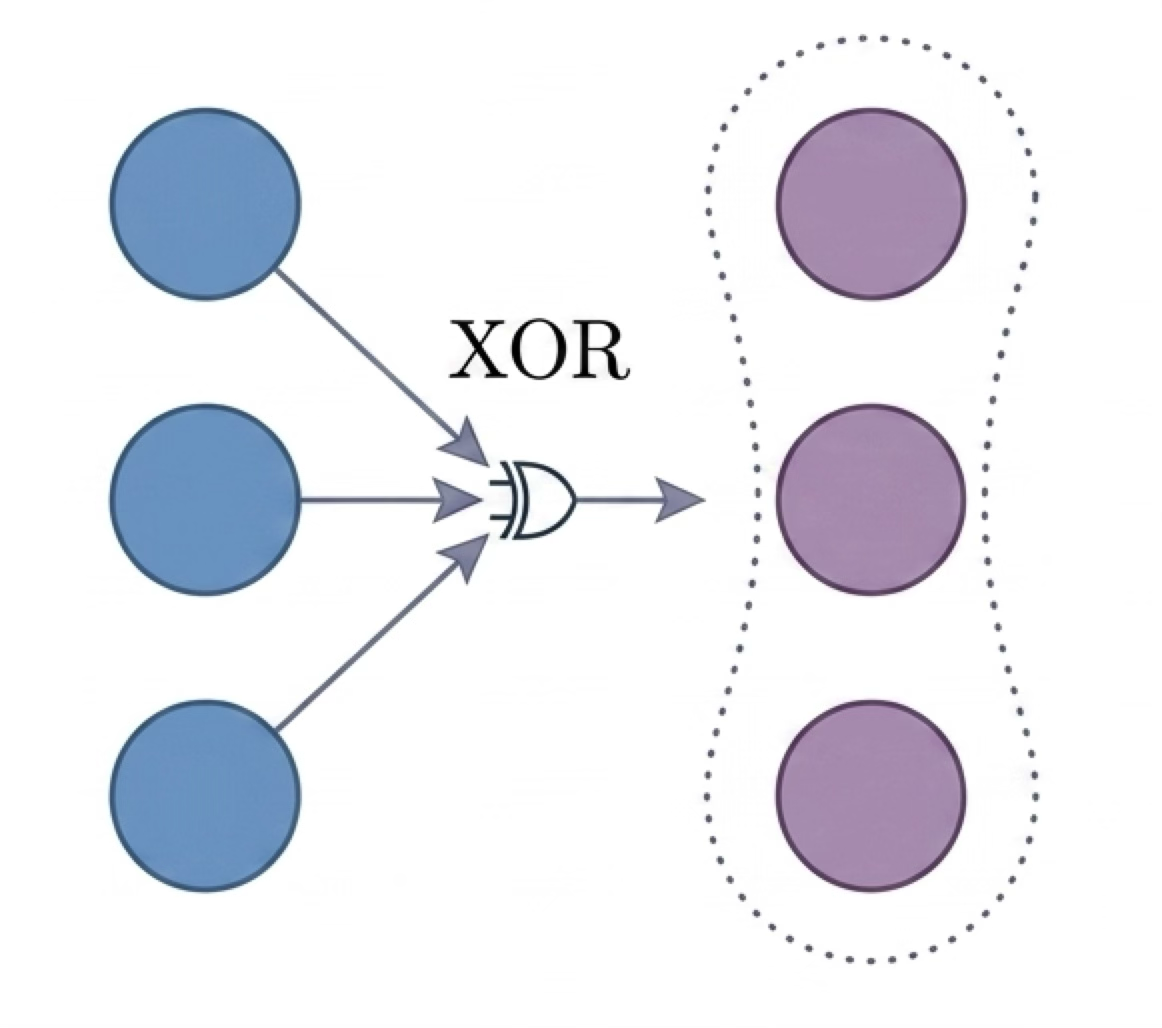}
    \caption{}
\end{subfigure}
\hfill
\begin{subfigure}[b]{0.31\textwidth}
    \centering
    \includegraphics[width=\linewidth]{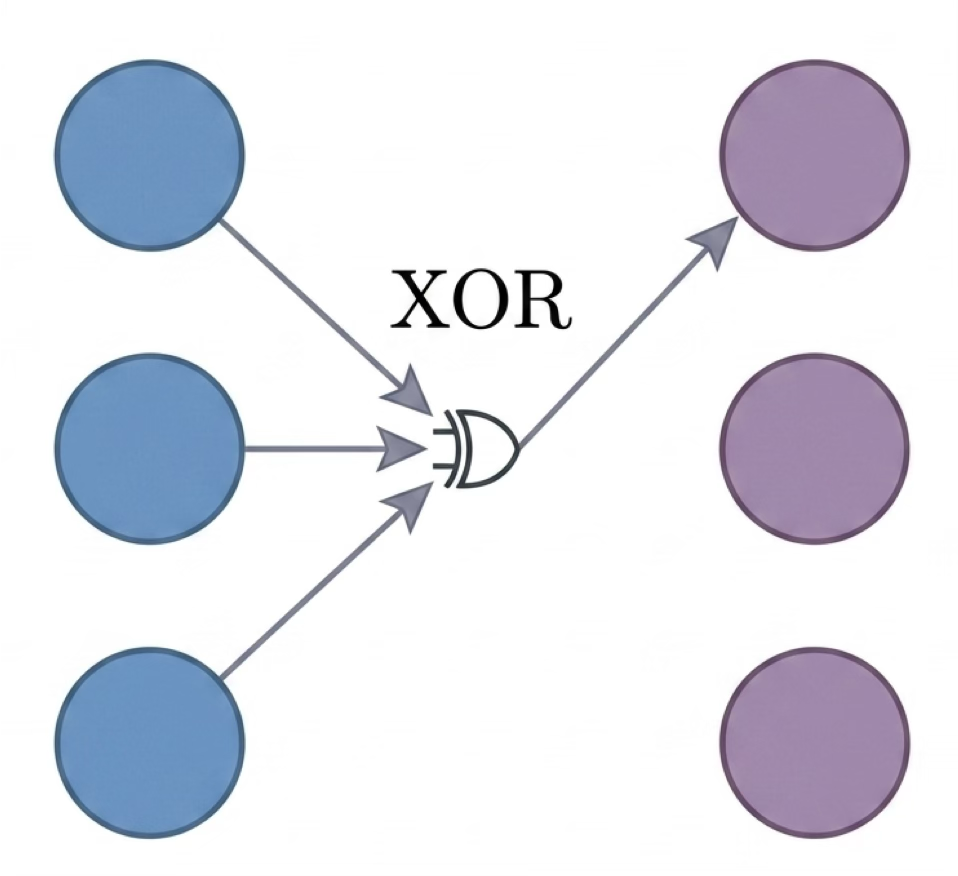}
    \caption{}
\end{subfigure}
\hfill
\begin{subfigure}[b]{0.32\textwidth}
    \centering
    \includegraphics[width=\linewidth,trim=0 0 12 0,clip]{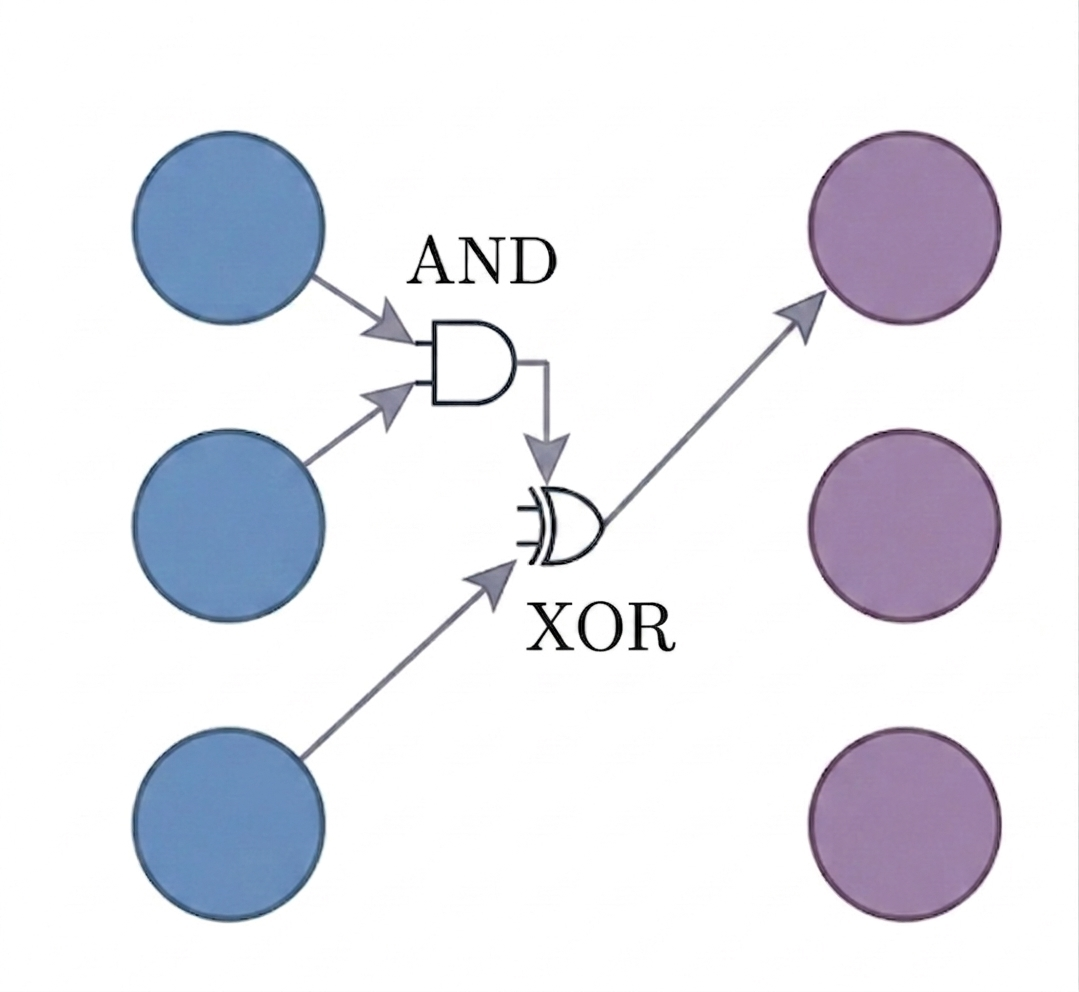}
    \caption{}
\end{subfigure}
\caption{Toy examples designed by Rosas et al. \cite{Rosas2020ReconcilingEA} and a discussion of the decomposition of downward causation. The circles from top to bottom represent $X^{(1)}_t$, $X^{(2)}_t$, and $X^{(3)}_t$, respectively. The blue and purple circles respectively denote the three microscopic variables in the dynamics at times $t$ and $t+1$. (a) Causal decoupling, (b) downward causation, and (c) downward causation can be further decomposed into flexibility and the synergistic influence of the environment on the subject. In all three examples, the overall EI is 1 bit.}
\label{fig:downward-causation-triptych}
\end{figure}

In addition to validating previous examples, we can further decompose downward causation. We can regard $X^{(1)}$ as the subject, and the other two variables as its environment. Then, according to Eq.~\ref{eq:downward_causation_decomposition}, each instance of downward causation can be decomposed into the subject's ability to respond to the environment (flexibility) and the synergistic influence of the environment on the subject (environment synergy). In Fig.~\ref{fig:downward-causation-triptych}(b), the environmental variables cannot individually exert any influence on $X^{(1)}_{t+1}$, so there is no environment synergy.

In Fig.~\ref{fig:downward-causation-triptych}(c), however, $X_{t+1}^{(1)} = \bigl(X_t^{(1)} \texttt{ AND } X_t^{(2)}\bigr) \oplus X_t^{(3)},$ and due to the presence of the AND gate, the environmental variables exert a stronger influence on $X^{(1)}_{t+1}$: when $X^{(2)}_{t}=0$, $X^{(1)}$ can no longer transmit its own information. Therefore, in this example, $DC_1 = 0.811$, where the flexibility of $X^{(1)}$ accounts for 0.5 and environment synergy accounts for 0.311.

\subsection{Synergy in Continuous Dynamics}
\label{sec:continuous_dynamics}

In addition to discrete Boolean systems, we further extend PEID to nonlinear dynamical systems in continuous state spaces. For continuous variables, EI and its synergistic decomposition require first estimating the relevant continuous probability densities from samples, where uniform sampling is performed over the input variables. To this end, this paper adopts transport map density estimation as the density estimation method for the continuous case \cite{baptista2024representation}, and uses it to compute the mutual information from joint source variables and single source variables to the target variable, thereby obtaining $EI$ and $Syn$ in continuous systems. The basic form of the transport map, the density transformation formula, and the implementation details for estimating continuous EI are provided in Appendix~\ref{app:transport-map-density-estimation}. The purpose of this section is to examine whether PEID can accurately characterize the corresponding dynamical structure in strongly nonlinear continuous systems. In the experimental setting, the independent uniform interventional distributions of the two source variables are fixed as
\begin{equation}
X_2^n, X_3^n \sim \mathrm{Uniform}\biggl[-\frac{L}{2}, \frac{L}{2}\biggr],
\qquad L = 2 .
\label{eq:tm-source-intervention}
\end{equation}
The known dynamics are
\begin{equation}
\begin{aligned}
X_1^{n+1}
&=
\alpha \sin\bigl(X_2^n X_3^n\bigr)
+
(1-\alpha)X_2^n
+
\sigma_{\epsilon}\epsilon_1^n,\\
X_2^{n+1}
&=
\sigma_{\epsilon}\epsilon_2^n,\\
X_3^{n+1}
&=
\sigma_{\epsilon}\epsilon_3^n,
\end{aligned}
\label{eq:tm-nonlinear-dynamics}
\end{equation}

where $\epsilon_1^n,\epsilon_2^n,\epsilon_3^n \overset{\mathrm{i.i.d.}}{\sim} \mathcal{N}(0,1).$ Here we are mainly concerned with the synergistic effect of $X_2^n, X_3^n$ on $X_1^{n+1}$. When the parameter $\alpha$ is large, the synergistic term $\sin(X_2^n X_3^n)$ becomes dominant. When $\alpha=0$, only the individual effect of $X_2^n$ remains in the dynamics, and no synergy is present.

Following the original definition of EI in Eq. \ref{eq:EI_def}, we compute $EI(X^n_2,X^n_3 \to X^n_1)$, $EI(X^n_2 \to X^n_1)$, and $EI(X^n_3 \to X^n_1)$, respectively. Correspondingly, the synergistic term is still defined as
\begin{equation}
Syn(X^n_2,X^n_3 \to X^n_1)
:=
EI(X^n_2,X^n_3 \to X^n_1)
-
EI(X^n_2 \to X^n_1)
-
EI(X^n_3 \to X^n_1).
\label{eq:tm-synergy}
\end{equation}

Fig.~\ref{fig:tm-alpha-noise-ei-decomposition} shows the EI decomposition results under different levels of external noise. Since $X_3^n$ has no independent effect term, $EI(X^n_3 \to X^n_1)$ remains close to zero over the entire parameter range. When $\alpha=0$, the target variable is entirely driven by the single-source term of $X_2^n$, so the joint EI is almost completely explained by the green single-source term, while the synergistic term is close to zero. As $\alpha$ increases, the single-source linear term $(1-\alpha)X_2^n$ gradually weakens, whereas the joint nonlinear term $\sin(X_2^nX_3^n)$ gradually becomes dominant. Correspondingly, $EI(X^n_2 \to X^n_1)$ decreases monotonically, while $Syn(X^n_2,X^n_3 \to X^n_1)$ continues to increase. When $\alpha=1$, the target is almost entirely determined by the joint effect of $X_2^n$ and $X_3^n$, and observing either source variable alone can hardly recover the target information; therefore, the joint EI is mainly contributed by the synergistic term. This conclusion holds under both low-noise and high-noise conditions.
\begin{figure}[htbp]
    \centering
    \begin{subfigure}[t]{0.48\textwidth}
        \centering
        \includegraphics[width=\textwidth]{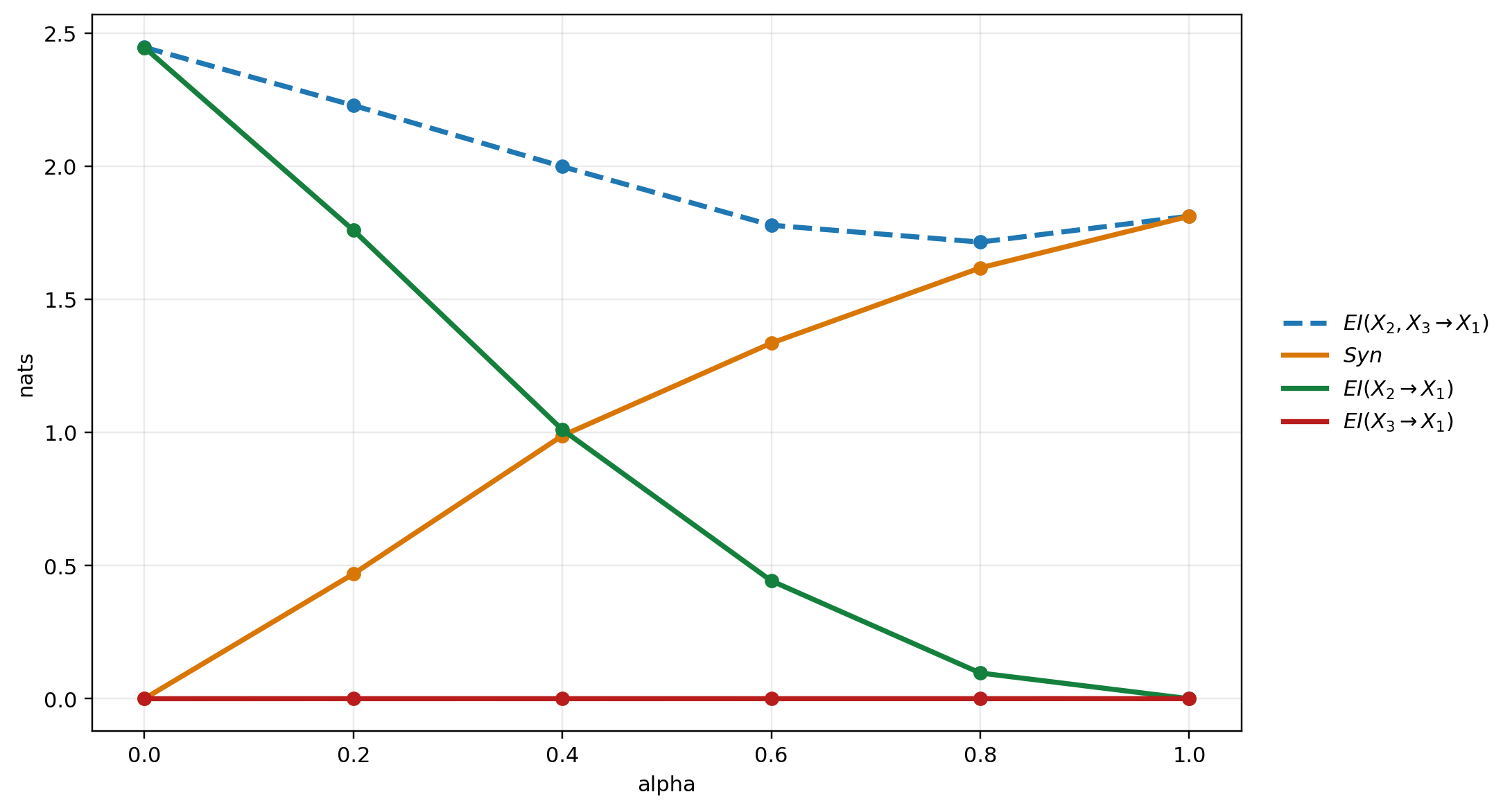}
        \caption{Low-noise condition, $\sigma_{\epsilon}=0.05$.}
        \label{fig:tm-alpha-low-noise-ei-decomposition}
    \end{subfigure}
    \hfill
    \begin{subfigure}[t]{0.48\textwidth}
        \centering
        \includegraphics[width=\textwidth]{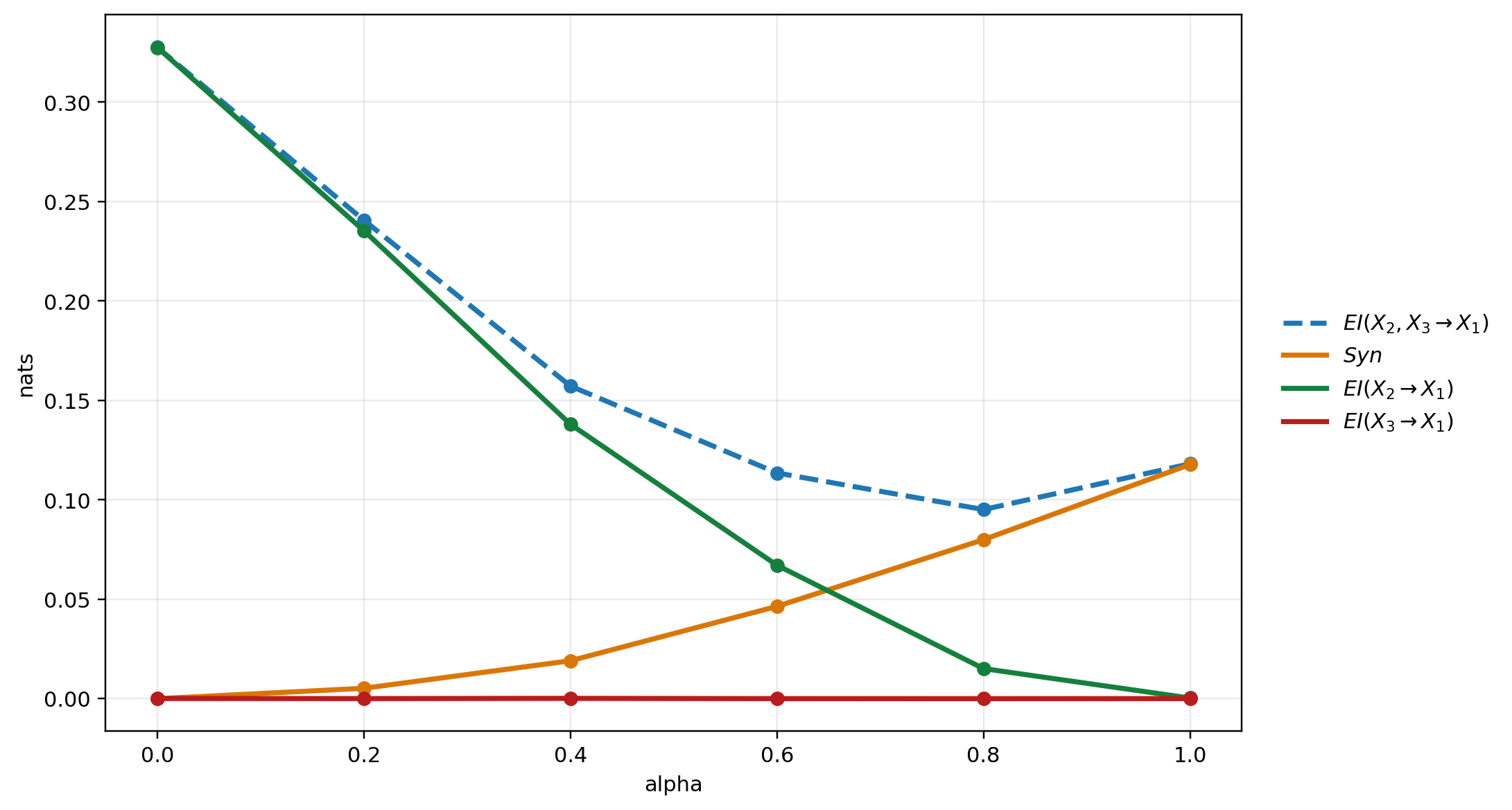}
        \caption{High-noise condition, $\sigma_{\epsilon}=0.6$.}
        \label{fig:tm-alpha-high-noise-ei-decomposition}
    \end{subfigure}
    \caption{EI decomposition results of the continuous nonlinear mechanism estimated by the transport map under different levels of external noise. The left panel corresponds to the low-noise condition, and the right panel corresponds to the high-noise condition.}
    \label{fig:tm-alpha-noise-ei-decomposition}
\end{figure}

Under the high-noise condition, the overall trend is still preserved, but the total EI becomes more sensitive to changes in the parameter $\alpha$. This indicates that, in this example, the synergistic effect in the dynamics is more easily affected by intrinsic noise than the single-variable unique effect. Nevertheless, the proportion of synergy in the joint EI remains stable.

This shows that, for strongly nonlinear continuous systems, continuous mutual information computed using transport-map-based density estimation can reliably recover the proportion of synergy. Its practical value lies in the fact that, when facing black-box machine learning models whose dynamical equations are difficult to reconstruct accurately, the causal relational structure in the dynamics can be computed through the PEID method.

\section{Application}
\label{sec:applications}

This section applies the continuous $EI$ decomposition method developed above to real-world air quality forecasting data. The experimental data are taken from KnowAir-V2, a high-quality, large-scale benchmark dataset for air quality forecasting, originally proposed together with the PCDCNet model to support the training and evaluation of deep-learning surrogate models that incorporate physical--chemical mechanisms \cite{wang2025knowair}. In this paper, we extract the Hangzhou station network from KnowAir-V2, set the prediction interval to $12h$, and, after fitting the dynamics with an MLP, simultaneously analyze the single-pollutant cross-station effects $\mathrm{O}_3 \to \mathrm{O}_3$ and $\mathrm{PM}_{2.5} \to \mathrm{O}_3$, as well as the bivariate synergistic effect of source-station O$_3$, PM$_{2.5}$ on target-station O$_3$. A detailed description of the data and machine-learning settings is provided in Section \ref{sec:air-pollution-experiment-details}.

The EI causal graph results are shown in Fig \ref{fig:hangzhou-air-ei-syn}.  Fig \ref{fig:hangzhou-air-ei-syn}(a) shows that the $12h$ $\mathrm{O}_3 \to \mathrm{O}_3$ effect in Hangzhou is not uniformly distributed across the entire network, but is dominated by a small number of strong edges. Station $1231A$ is particularly prominent, pointing to multiple target stations including $3558A$, $3557A$, $1223A$, $1228A$, and $1224A$. This indicates that, in the learned predictor, perturbations to input O$_3$ at 1231A propagate to multiple downstream stations. This pattern is consistent with previous observations of an east-higher--west-lower spatial gradient of ozone across Hangzhou stations, with the northeastern area exhibiting relatively elevated ozone levels \cite{chen2026hangzhouozone}. In the learned predictor, the outgoing edges from 1231A are therefore compatible with, but do not by themselves demonstrate, its possible role as an upwind ozone source under certain circulation conditions. Moreover, hourly local circulations in the Hangzhou metropolitan area, including sea--land breezes, urban heat-island circulation, and valley circulations, can mediate inter-station ozone redistribution \cite{han2023coastallocalcirculations}. Station $3558A$ (Lin'an Municipal Government Building \cite{wang2025knowair}) is located in a foothill basin-like environment, surrounded by mountains on three sides, where particulate matter and meteorological factors exhibit obvious lagged relationships \cite{lin2019linanpollution}. Therefore, in the learned predictor, $3558A$ may appear both as a receptor-like node and as an intermediate node whose $\mathrm{O}_3$ input contains terrain- and circulation-related delayed information relevant to nearby stations such as $1227A$.

\begin{figure}[htbp]
\centering
\includegraphics[width=\textwidth]{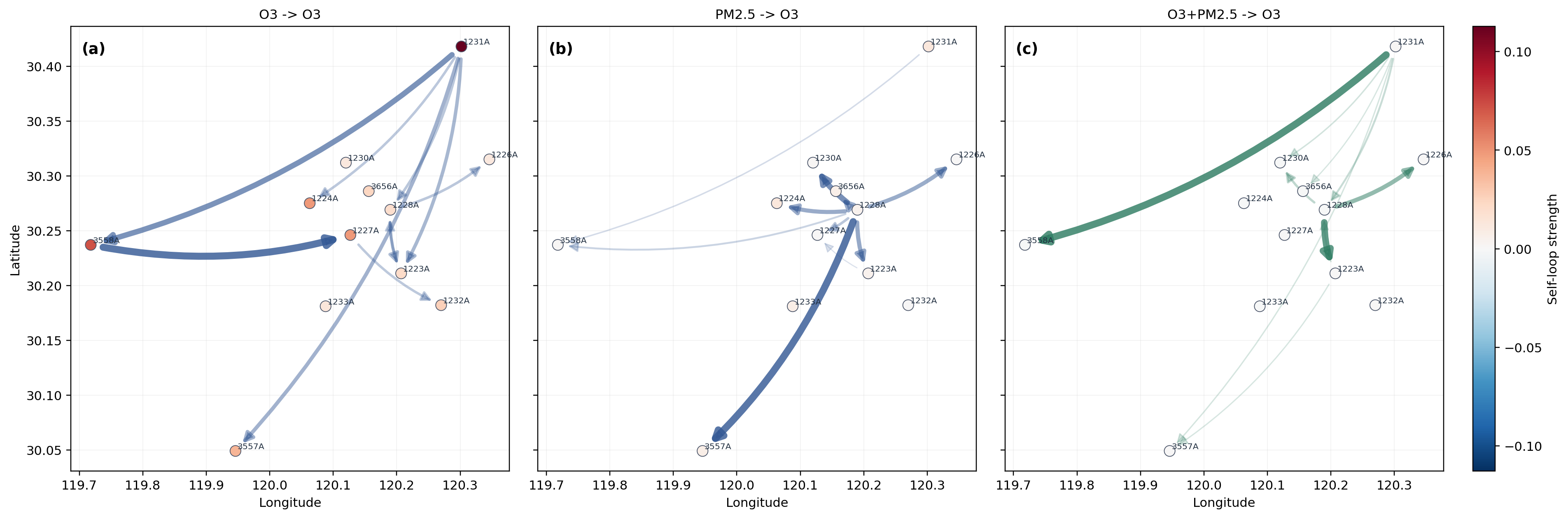}
\caption{Model-based inter-station $EI$ and $Syn$ structures on the Hangzhou $12\,\mathrm{h}$ air-quality forecasting task. (a) shows the $\mathrm{O}_3 \to \mathrm{O}_3$ pairwise graph, (b) shows the $\mathrm{PM}_{2.5} \to \mathrm{O}_3$ pairwise graph, and (c) shows the $\mathrm{O}_3+\mathrm{PM}_{2.5} \to \mathrm{O}_3$ synergy graph. Nodes are placed at their real longitude--latitude locations, and node colors indicate self-loop strength; arrows indicate cross-station edges, and line widths indicate the relative strength of each type of edge. The edges displayed in the figure are the top-10 edges selected according to their magnitudes.}
\label{fig:hangzhou-air-ei-syn}
\end{figure}

Fig \ref{fig:hangzhou-air-ei-syn}(b) presents the causal structure of $\mathrm{PM}_{2.5} \to \mathrm{O}_3$. Here, the strongest edges are highly concentrated among the outgoing edges from $1228A$ (Zhejiang Agricultural University). Combined with the urban-core context of this station, this pattern suggests that the learned predictor may capture information associated with traffic- and industry-related precursor conditions, rather than a direct physical causal pathway from $\mathrm{PM}_{2.5}$ to $\mathrm{O}_3$. Urban-core traffic, industrial sources, solvent use, and combustion sources jointly affect VOCs, NOx, secondary aerosols, and $\mathrm{PM}_{2.5}$, and these variables in turn jointly modify the formation and transport conditions of O$_3$ \cite{wang2024hangzhouvocs,tao2024pm25vocs}. Therefore, when the prediction window is extended to $12h$ or longer, the PM$_{2.5}$ variable at $1228A$ becomes an important factor in changes in ozone concentration.

Fig \ref{fig:hangzhou-air-ei-syn}(c) shows the bivariate synergistic term. Its pattern is similar to that in panel (b), both reflecting the importance of $1228A$ (Zhejiang Agricultural University) as a source station. This indicates that the coordinated control of PM$_{2.5}$ and O$_3$ is an important issue, especially for ozone pollution caused by urban-core traffic. They also exhibit important differences. For example, $1228A \to 3557A$ is strong in the $\mathrm{PM}_{2.5} \to \mathrm{O}_3$ causal graph, but is almost absent in the synergy graph. This indicates that, in the learned predictor, PM$_{2.5}$ at 1228A contributes a non-redundant unique predictive component to forecasted O$_3$ at 3557A, whereas O$_3$ at 1228A contributes essentially none. This is likely because urban tropospheric $\mathrm{O}_3$ has a relatively short photochemical lifetime---on the order of hours near sources to one to two days under transport conditions---so its cross-station information flow depends strongly on wind fields, local circulation, and photochemical conditions \cite{dewan2022tropospheric}.

It is worth emphasizing that the EI and Syn quantities reported in this section are properties of the trained MLP under input-side maximum-entropy interventions. Their interpretation as causal quantities of the underlying atmospheric system requires the strong assumption of causal sufficiency. In the present experiment, we use 10 input variables, including two pollutants and eight meteorological variables, whereas KnowAir-V2 also provides MEIC anthropogenic emission inventories, including NOx, VOC, SO$_2$, NH$_3$, primary PM$_{2.5}$, and PM$_{10}$. These emission variables are direct exogenous drivers of both PM$_{2.5}$ and O$_3$, and any inter-station edge mediated by shared upwind emissions may be misattributed under the present configuration. Therefore, the results below should be interpreted as intervention-based responses and information-flow structures of the learned predictor, rather than as direct evidence of physical causal effects in the atmospheric system.

\section{Discussion}
\label{sec:discussion}

This paper proposes a unified analytical framework based on effective information decomposition for studying the relationship among causal emergence, effective information, and synergistic information. The central idea is that, under an interventionist view of causation, the dynamical mechanism that maps a system from its current state to its future state can be regarded as a causal channel, and the effective information contributions of different source variables, groups of variables, or macroscopic variables to the target state can be evaluated under a maximum-entropy interventional distribution. Based on this idea, this paper defines partial effective information decomposition (PEID), thereby further decomposing effective information, which has traditionally been discussed at the whole-system scale, into single-variable effects, joint-variable effects, and synergistic effects. As a result, we can not only determine whether a system has higher effective information at a macroscopic scale, but also further ask whether the source of this macroscopic effective information lies in single-variable transmission, within-module integration, or irreducible synergistic causal effects among multiple variables.

The first main contribution of this paper is to connect the interventionist problem of causal emergence with the problem of information decomposition. Existing studies of causal emergence usually focus on whether a macroscopic scale can obtain higher effective information than the microscopic scale \cite{Hoel2013QuantifyingCE,Hoel2017WhenTM,Rosas2020ReconcilingEA}, whereas PID and related information-decomposition studies more often focus on how the informational contribution of multiple source variables to a target variable can be decomposed into redundant, unique, and synergistic components \cite{Williams2010NonnegativeDO,Bertschinger2013SharedInformation,Harder2013BivariateRedundancy,Finn2018PointwisePID}. This paper shows that, under independent maximum-entropy interventions, the redundancy among source variables can be systematically eliminated, so that the non-additive component of effective information naturally corresponds to a form of synergistic causal information in the interventionist sense. This provides a more direct dynamical explanation for understanding the idea that ``the whole is greater than the sum of its parts'': when the effective information from joint source variables to a future state strictly exceeds the sum of the contributions from each part alone, the system contains a synergistic causal structure that cannot be reduced to pairwise effects.

Furthermore, we propose an EI-based causal graph representation. Traditional causal graphs usually take pairwise edges as their basic units, but many causal effects in complex systems cannot be fully expressed by binary edges. By representing pairwise unique effective information as ordinary directed edges and synergistic effective information as hyperedges, this paper enables causal graphs to express both single-variable causal effects and higher-order synergistic causal effects. We can therefore extend this framework to multiscale causal graphs and coarse-graining analysis. Through a handcrafted six-node Boolean network, this paper shows that complex higher-order dependencies at the microscopic level can be absorbed into the internal encoding of macroscopic nodes after an appropriate coarse-graining, thereby forming a simpler causal graph at the macroscopic scale. This result indicates that causal emergence does not only mean that macroscopic variables have higher EI in numerical value, but also that a macroscopic description may reorganize dispersed, complex, and higher-order causal structures at the microscopic level into a lower-dimensional, more stable, and more interpretable dynamical mechanism. Therefore, the advantage of the macroscopic scale can be understood as a form of causal-structure compression: it preserves the most effective information about future states while reducing redundancy and unnecessary complexity in the microscopic description. At the same time, we examine the definition and computation of downward causation under PEID.

PEID can be applied not only to discrete systems, but also to continuous nonlinear dynamical systems. This paper introduces transport map density estimation \cite{baptista2024representation}, and demonstrates through nonlinear continuous-dynamics experiments that this method can reliably distinguish single-variable driving from synergistic driving under different noise levels. This shows that PEID can be used for black-box dynamical analysis in continuous state spaces. We then apply this framework to a real-world air quality forecasting task. By training an MLP on the KnowAir-V2 dataset \cite{wang2025knowair} and using the transport map to estimate inter-station $EI$ and $Syn$ on the fixed predictor, this paper demonstrates that PEID can serve as a causal interpretability tool for machine-learning forecasting models. In the Hangzhou $12h$ air pollution experiment, the $\mathrm{O}_3 \to \mathrm{O}_3$ graph reveals several stable outward ozone hubs and terrain-filtered delayed propagation relations; the $\mathrm{PM}_{2.5} \to \mathrm{O}_3$ graph shows that the mixed-pollution background in the urban core has important explanatory power for future ozone changes; and the $\mathrm{O}_3+\mathrm{PM}_{2.5} \to \mathrm{O}_3$ synergy graph further indicates that bivariate synergistic effects do exist, but are weaker than the dominant pairwise cross-station structure in the current task. This application shows that the proposed method can read out causal information structures with spatial interpretability from learned machine-learning parameters without directly resolving the original physical equations.

Nevertheless, this paper still has several limitations. First, the theoretical derivations in this paper rely on maximum-entropy interventions. In real observational data, we must provide reasonable assumptions that allow us to make the most appropriate maximum-entropy interventions, because variables in real systems are often constrained by physical constraints, conservation relations, common drivers, and unobserved confounders. Future work can further investigate how to formulate reasonable assumptions about maximum-entropy distribution interventions given observational data and prior knowledge.

Second, this paper only shows that maximum-entropy distribution interventions can eliminate redundant information on the source-variable side, but they cannot eliminate redundancy on the target-variable side; therefore, no more concise and reasonable definition is provided for causal information decomposition on the target-variable side. Future research may start from EI and attempt to provide a more fine-grained partition into informational atoms, especially by extending the decomposition to target variables. In this case, we would be able to observe redundant information in causal mechanisms caused by one-to-many mappings. In addition, for PEID with continuous variables, we cannot guarantee the non-negativity of synergistic effective information, and further research is needed to verify the rigorous mathematical definition of PEID for continuous variables.

Moreover, the current dynamical systems all assume Markovianity. How to perform reasonable maximum-entropy distribution interventions and PEID analysis on a time-delayed non-Markovian system is a very important problem. Solving this theoretical problem would allow us to use architectures with stronger fitting capabilities, such as RNNs and Transformers, to address time-series prediction problems in real data, thereby obtaining more comprehensive and reliable causal structures. In addition, we can combine PEID with machine-learning frameworks for identifying causal emergence \cite{Yang2025FindingEI}, and provide analyses of multiscale and cross-scale causal structures. This would help us understand, predict, and control complex systems more concisely and effectively.

Overall, this paper provides a new method for quantifying the causal structure of complex systems starting from effective information. It unifies causal emergence, higher-order synergy, downward causation, multiscale coarse-graining, and machine-learning model interpretation within the same interventionist information-theoretic framework. The experiments in this paper show that PEID can not only characterize the causal structures of known mechanisms in controlled Boolean systems and continuous nonlinear systems, but also reveal physically interpretable inter-station causal influences in a real-world air quality forecasting task. If its statistical estimation, theoretical axiomatization, and large-scale machine-learning algorithms can be further improved in the future, PEID is expected to become a general tool for analyzing multiscale causal structures and emergent mechanisms in complex systems.

\bibliographystyle{unsrt}  
\bibliography{references}

\appendix

\section{PID axiomatic background}
\label{app:pid-axioms}

This appendix briefly reviews the axiomatic language of partial information decomposition (PID), which will be used to justify the correspondence between PEID and PID in the three-variable case.

For two source variables $X_1,X_2$ and one target variable $Y$, the joint mutual information is decomposed as
\begin{equation}
I(X_1,X_2;Y)
=
Red(X_1,X_2;Y)
+
Un(X_1;Y\mid X_2)
+
Un(X_2;Y\mid X_1)
+
Syn(X_1,X_2;Y).
\label{eq:pid_decomp_app}
\end{equation}

Here $Red(\cdot)$ denotes the redundancy term, while the unique and synergistic terms are induced by the redundancy function. In this paper, we only use the following standard axioms.

\paragraph{Symmetry (S).}
The redundancy is invariant under permutations of the source variables:
\begin{equation}
Red(A_1,\dots,A_k;S)
=
Red(A_{\pi(1)},\dots,A_{\pi(k)};S).
\label{eq:pid_symmetry}
\end{equation}

\paragraph{Self-redundancy (I).}
For a single source, redundancy reduces to ordinary mutual information:
\begin{equation}
Red(A;S)=I(A;S).
\label{eq:pid_self_redundancy}
\end{equation}

\paragraph{Monotonicity (M).}
Redundancy cannot increase when more sources are included:
\begin{equation}
Red(A_1,\dots,A_k;S)
\le
Red(A_1,\dots,A_{k-1};S).
\label{eq:pid_monotonicity}
\end{equation}

\paragraph{Identity (Id).}
When the target is the union of the two sources,
\begin{equation}
Red(A_1,A_2;A_1\cup A_2)=I(A_1;A_2).
\label{eq:pid_identity}
\end{equation}

\paragraph{Left chain rule (LC).}
For two target variables $S$ and $S'$ and source variables $A_1,\dots,A_k$, the redundancy associated with the joint target $(S,S')$ satisfies
\begin{equation}
Red(A_1,\dots,A_k;S,S')
=
Red(A_1,\dots,A_k;S)
+
Red(A_1,\dots,A_k;S'\mid S).
\label{eq:pid_left_chain_rule}
\end{equation}
Here the conditional redundancy term is defined by averaging the redundancy under the conditional distributions:
\begin{equation}
Red(A_1,\dots,A_k;S'\mid S)
:=
\sum_s p(s)\,
Red_{p(\cdot\mid s)}(A_1,\dots,A_k;S').
\label{eq:pid_conditional_redundancy}
\end{equation}
This axiom can be understood as a redundancy-level analogue of the chain rule for mutual information. It specifies how shared information about a composite target can be decomposed into shared information about the first target component and additional shared information about the second target component conditioned on the first. 

Under these axioms, the unique term is determined by
\begin{equation}
Un(X_1;Y\mid X_2)
=
I(X_1;Y)-Red(X_1,X_2;Y),
\label{eq:pid_unique_app}
\end{equation}
while the synergy term is then fixed by Eq.~\eqref{eq:pid_decomp_app}.

\section{Compatibility with PID in the three-variable case}
\label{app:three_variable_proof}

We now consider the basic $2\to1$ setting with source variables $X_t^{(1)},X_t^{(2)}$ and target variable $X_{t+1}$. Under the source-side maximum-entropy independent intervention,
\begin{equation}
X_t^{(1)} \perp X_t^{(2)}.
\label{eq:appendix_independence}
\end{equation}

\begin{lemma}
Under the PID axioms $(S,I,M,LC,Id)$, for arbitrary random variables $U,V,T$, if $U \perp V$, then
\begin{equation}
Red(U,V;T)=0.
\label{eq:lemma_redundancy_zero}
\end{equation}
\end{lemma}

\begin{proof}
By the target chain rule axiom $(LC)$, for the augmented target $(U,V,T)$ we have
\begin{equation}
Red(U,V;U,V,T)
=
Red(U,V;U,V)
+
Red(U,V;T\mid U,V),
\label{eq:lc_expand_1}
\end{equation}
and also
\begin{equation}
Red(U,V;U,V,T)
=
Red(U,V;T)
+
Red(U,V;U,V\mid T).
\label{eq:lc_expand_2}
\end{equation}

Comparing Eqs.~\eqref{eq:lc_expand_1} and \eqref{eq:lc_expand_2}, we obtain
\begin{equation}
Red(U,V;U,V)
+
Red(U,V;T\mid U,V)
=
Red(U,V;T)
+
Red(U,V;U,V\mid T).
\label{eq:lc_compare}
\end{equation}

By the identity axiom $(Id)$,
\begin{equation}
Red(U,V;U,V)=I(U;V).
\label{eq:identity_uv}
\end{equation}
Since $U \perp V$, it follows that
\begin{equation}
Red(U,V;U,V)=0.
\label{eq:identity_zero}
\end{equation}

Moreover, by monotonicity and nonnegativity,
\begin{equation}
0 \le Red(U,V;T\mid U,V) \le I(U,V;T\mid U,V)=0,
\label{eq:conditional_zero}
\end{equation}
hence
\begin{equation}
Red(U,V;T\mid U,V)=0.
\label{eq:conditional_redundancy_zero}
\end{equation}

Substituting Eqs.~\eqref{eq:identity_zero} and \eqref{eq:conditional_redundancy_zero} into Eq.~\eqref{eq:lc_compare} yields
\begin{equation}
Red(U,V;T)=0.
\end{equation}
This proves the lemma.
\end{proof}

Applying Lemma~1 to Eq.~\eqref{eq:appendix_independence}, we immediately obtain
\begin{equation}
Red\bigl(X_t^{(1)},X_t^{(2)};X_{t+1}\bigr)=0.
\label{eq:red_zero_xt}
\end{equation}

Therefore,
\begin{equation}
Un\bigl(X_t^{(1)};X_{t+1}\mid X_t^{(2)}\bigr)
=
I\bigl(X_t^{(1)};X_{t+1}\bigr),
\label{eq:un_1_xt}
\end{equation}
and similarly,
\begin{equation}
Un\bigl(X_t^{(2)};X_{t+1}\mid X_t^{(1)}\bigr)
=
I\bigl(X_t^{(2)};X_{t+1}\bigr).
\label{eq:un_2_xt}
\end{equation}

Using the PID decomposition identity,
\begin{equation}
\begin{aligned}
I\bigl(X_t^{(1)},X_t^{(2)};X_{t+1}\bigr)
=\;&
Red\bigl(X_t^{(1)},X_t^{(2)};X_{t+1}\bigr) \\
&+
Un\bigl(X_t^{(1)};X_{t+1}\mid X_t^{(2)}\bigr) \\
&+
Un\bigl(X_t^{(2)};X_{t+1}\mid X_t^{(1)}\bigr) \\
&+
Syn\bigl(X_t^{(1)},X_t^{(2)};X_{t+1}\bigr),
\end{aligned}
\label{eq:pid_decomp_xt}
\end{equation}
and substituting Eqs.~\eqref{eq:red_zero_xt}--\eqref{eq:un_2_xt}, we obtain
\begin{equation}
Syn\bigl(X_t^{(1)},X_t^{(2)};X_{t+1}\bigr)
=
I\bigl(X_t^{(1)},X_t^{(2)};X_{t+1}\bigr)
-
I\bigl(X_t^{(1)};X_{t+1}\bigr)
-
I\bigl(X_t^{(2)};X_{t+1}\bigr).
\label{eq:syn_difference_xt}
\end{equation}

On the other hand, under the partition $\mathcal{P}=\{X_t^{(1)},X_t^{(2)}\}$, the PEID synergy term is defined by
\begin{equation}
Syn_{\mathcal{P}}^{\mathrm{EID}}\bigl((X_t^{(1)},X_t^{(2)}) \to X_{t+1}\bigr)
=
I\bigl(X_t^{(1)},X_t^{(2)};X_{t+1}\bigr)
-
I\bigl(X_t^{(1)};X_{t+1}\bigr)
-
I\bigl(X_t^{(2)};X_{t+1}\bigr).
\label{eq:peid_syn_xt}
\end{equation}

Comparing Eqs.~\eqref{eq:syn_difference_xt} and \eqref{eq:peid_syn_xt}, we conclude that
\begin{equation}
Syn_{\mathcal{P}}^{\mathrm{EID}}\bigl((X_t^{(1)},X_t^{(2)}) \to X_{t+1}\bigr)
=
Syn\bigl(X_t^{(1)},X_t^{(2)};X_{t+1}\bigr).
\label{eq:peid_pid_equiv}
\end{equation}

Hence, in the basic $2\to1$ case, once the source-side intervention renders the two source variables independent, the PID redundancy term vanishes and the PEID synergy reduces exactly to the PID synergy term.

\section{Proof of Theorem~\ref{thm:nonnegativity}}
\label{app:nonnegativity-proof}

The proof strategy follows the discussion of the nonnegativity of $\Phi$ in IIT~2.0 by Oizumi et al.~\cite{Oizumi2016MeasuringII}.

By definition,
\begin{equation}
Syn_{\mathcal{P}}^{\mathrm{EID}}\bigl(X_t^A \to X_{t+1}^B\bigr)
=
EI\bigl(X_t^A \to X_{t+1}^B\bigr)
-
\sum_{i=1}^m EI\bigl(M_i \to X_{t+1}^B\bigr).
\label{eq:A1}
\end{equation}

Expanding the mutual informations in terms of entropy gives
\begin{equation}
\begin{aligned}
Syn_{\mathcal{P}}^{\mathrm{EID}}\bigl(X_t^A \to X_{t+1}^B\bigr)
=\;&
H_{q^{\max}}\bigl(X_t^A\bigr)
-
H_{q^{\max}}\bigl(X_t^A \mid X_{t+1}^B\bigr) \\
&-
\sum_{i=1}^m H_{q^{\max}}(M_i)
+
\sum_{i=1}^m H_{q^{\max}}\bigl(M_i \mid X_{t+1}^B\bigr).
\end{aligned}
\label{eq:A2}
\end{equation}

Since the source-side maximum-entropy intervention factorizes over the partition $\mathcal{P}$,
\begin{equation}
q^{\max}(x)=\prod_{i=1}^m q^{\max}(m_i),
\label{eq:A3}
\end{equation}
the total entropy of the source variables satisfies
\begin{equation}
H_{q^{\max}}\bigl(X_t^A\bigr)=\sum_{i=1}^m H_{q^{\max}}(M_i).
\label{eq:A4}
\end{equation}

Substituting Eq.~\eqref{eq:A4} into Eq.~\eqref{eq:A2}, we obtain
\begin{equation}
Syn_{\mathcal{P}}^{\mathrm{EID}}\bigl(X_t^A \to X_{t+1}^B\bigr)
=
\sum_{i=1}^m H_{q^{\max}}\bigl(M_i \mid X_{t+1}^B\bigr)
-
H_{q^{\max}}\bigl(X_t^A \mid X_{t+1}^B\bigr).
\label{eq:A5}
\end{equation}

The right-hand side is exactly the definition of conditional total correlation. Therefore,
\begin{equation}
Syn_{\mathcal{P}}^{\mathrm{EID}}\bigl(X_t^A \to X_{t+1}^B\bigr)
=
TC_{q^{\max}}\bigl(M_1,\dots,M_m \mid X_{t+1}^B\bigr).
\label{eq:A6}
\end{equation}

Furthermore, the conditional total correlation can be written as the expectation of a conditional Kullback--Leibler divergence:
\begin{equation}
TC_{q^{\max}}\bigl(M_1,\dots,M_m \mid X_{t+1}^B\bigr)
=
\mathbb{E}_{q^{\max}(x_{t+1}^B)}
\Biggl[
D_{\mathrm{KL}}
\Biggl(
q^{\max}\bigl(m_1,\dots,m_m \mid x_{t+1}^B\bigr)
\,\Big\|\,
\prod_{i=1}^m q^{\max}\bigl(m_i \mid x_{t+1}^B\bigr)
\Biggr)
\Biggr].
\label{eq:A7}
\end{equation}

Since the Kullback--Leibler divergence is always nonnegative, we have
\begin{equation}
TC_{q^{\max}}\bigl(M_1,\dots,M_m \mid X_{t+1}^B\bigr)\ge 0.
\label{eq:A8}
\end{equation}

Hence,
\begin{equation}
Syn_{\mathcal{P}}^{\mathrm{EID}}\bigl(X_t^A \to X_{t+1}^B\bigr)\ge 0.
\label{eq:A9}
\end{equation}

This also gives the mechanistic interpretation of the synergistic term: once the target subset $X_{t+1}^B$ is given, if the source parts still fail to be conditionally independent, then the remaining conditional coupling is precisely what constitutes the source-side synergy in this paper.

\section{Proof of Theorem~\ref{thm:hierarchicaladditivity}}
\label{app:hierarchical_proof}

Let the original partition be $\mathcal{P}$, and suppose that one block $M^\star \in \mathcal{P}$ is refined into
\begin{equation}
\mathcal{R} = \{R_1, \dots, R_k\}, \qquad
\mathcal{P}' = \bigl(\mathcal{P} \setminus \{M^\star\}\bigr) \cup \mathcal{R}.
\end{equation}

By the definition of synergy, we have
\begin{equation}
Syn_{\mathcal{P}'}^{\mathrm{EID}}\bigl(X_t^A \to X_{t+1}^B\bigr)
=
EI\bigl(X_t^A \to X_{t+1}^B\bigr)
-
\sum_{M \in \mathcal{P} \setminus \{M^\star\}} EI\bigl(M \to X_{t+1}^B\bigr)
-
\sum_{r=1}^k EI\bigl(R_r \to X_{t+1}^B\bigr).
\end{equation}

Adding and subtracting $EI\bigl(M^\star \to X_{t+1}^B\bigr)$ on the right-hand side yields
\begin{equation}
\begin{aligned}
Syn_{\mathcal{P}'}^{\mathrm{EID}}\bigl(X_t^A \to X_{t+1}^B\bigr)
&=
\Bigl[
EI\bigl(X_t^A \to X_{t+1}^B\bigr)
-
\sum_{M \in \mathcal{P} \setminus \{M^\star\}} EI\bigl(M \to X_{t+1}^B\bigr)
-
EI\bigl(M^\star \to X_{t+1}^B\bigr)
\Bigr] \\
&\quad
+
\Bigl[
EI\bigl(M^\star \to X_{t+1}^B\bigr)
-
\sum_{r=1}^k EI\bigl(R_r \to X_{t+1}^B\bigr)
\Bigr].
\end{aligned}
\end{equation}

The first term is precisely $Syn_{\mathcal{P}}^{\mathrm{EID}}\bigl(X_t^A \to X_{t+1}^B\bigr)$, while the second term equals $Syn_{\mathcal{R}}^{\mathrm{EID}}\bigl(M^\star \to X_{t+1}^B\bigr)$. Therefore,
\begin{equation}
Syn_{\mathcal{P}'}^{\mathrm{EID}}\bigl(X_t^A \to X_{t+1}^B\bigr)
=
Syn_{\mathcal{P}}^{\mathrm{EID}}\bigl(X_t^A \to X_{t+1}^B\bigr)
+
Syn_{\mathcal{R}}^{\mathrm{EID}}\bigl(M^\star \to X_{t+1}^B\bigr).
\end{equation}

This identity shows that the synergistic term admits a consistent additive decomposition under partition refinement, and hence can be propagated and accumulated across hierarchical levels.

\section{Supplementary Experiment on Non-Optimal Coarse-Graining}
\label{sec:exp3app}

As a supplement to Fig.~7.4(b) in Sec.~6.4, we next present the macro--micro causal graph comparison for the same hand-crafted Boolean system under the ``representative non-optimal coarse-graining''. The corresponding partition is $\{\{a1,a2\};\{b1,c1\};\{b2,c2\}\}$, which mixes together micro variables that originally belong to different macro blocks. It can therefore be used to directly illustrate how an ``incorrect partition'' fails to absorb the local control synergies inside those blocks into macro nodes, and instead regenerates them as pseudo-hyperedges at the macro level.

\begin{figure}[htbp]
\centering
\includegraphics[width=0.90\linewidth]{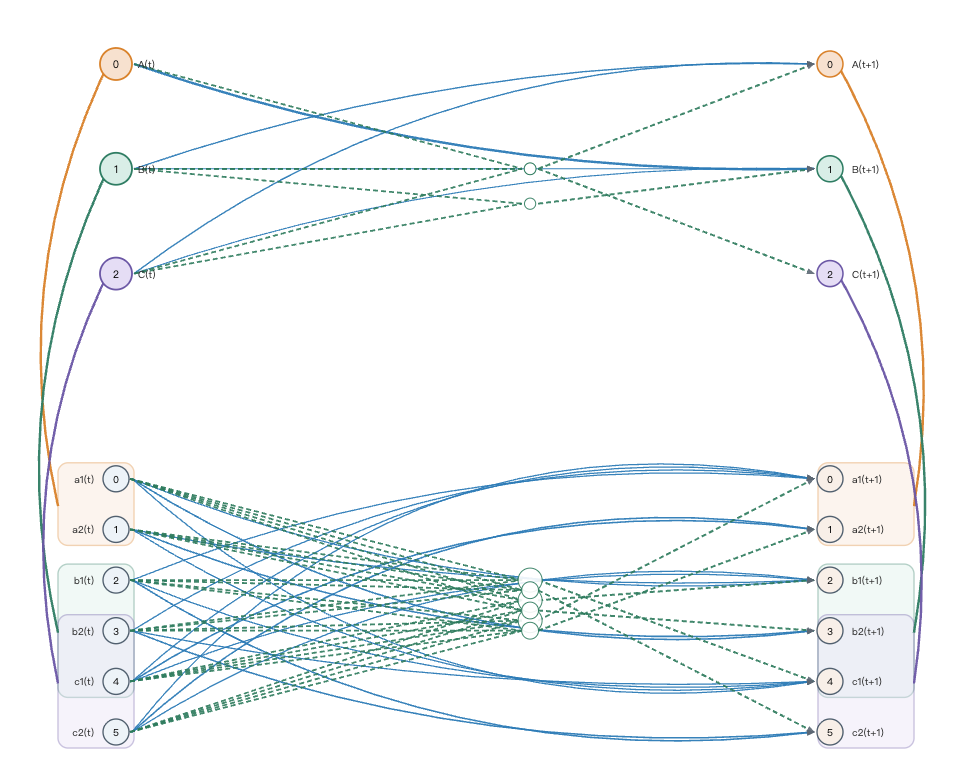}
\caption{Macro--micro causal graph comparison under a representative non-optimal coarse-graining. The upper half shows the macro EI causal graph, the lower half shows the micro EI causal graph, and the colored curves on the left and right indicate the coarse-graining maps at times $t$ and $t+1$, respectively.}
\label{fig:non-optimal-macro-comparison}
\end{figure}

This figure is consistent with the current notebook output: under this representative non-optimal coarse-graining, the macro total EI is only $1.658$, which is significantly lower than the $3.000$ achieved by the optimal coarse-graining; at the same time, the macro total synergy rises to $0.393$, whereas the optimal coarse-graining attains $0.000$. At the macro level in the figure, clear synergistic hyperedges reappear, and the strongest one reaches $0.150$ bit; correspondingly, this partitioning also supports the claim made in Sec.~6.4 that, when EI is not maximized, the apparent pairwise mechanism at the macro level is no longer clearer, but instead the incorrect coarse-graining creates additional local high-order dependencies.

\section{Transport map density estimation for continuous EI and synergy}
\label{app:transport-map-density-estimation}

This appendix summarizes the transport-map density estimator used for the continuous-system experiments. In discrete systems, EI and synergy can be computed from the transition probability matrix. In continuous systems, however, the relevant quantities are expressed through mutual information under intervention distributions, and therefore require estimates of joint and marginal densities. The role of the transport-map estimator is to provide a sample-based density model for this purpose. It is not intended to replace the local Jacobian--Covariance analysis of the mechanism, but rather to give a distribution-level estimator of the same intervention-based quantities.

Let $\mathbf{z}\in\mathbb{R}^{d}$ denote a concatenated sample vector, which may contain source variables, target variables, or both. A monotone triangular transport map is a differentiable transformation
\begin{equation}
\mathbf{T}(\mathbf{z})
=
\bigl(T_{1}(z_{1}),T_{2}(z_{1},z_{2}),\ldots,T_{d}(z_{1},\ldots,z_{d})\bigr),
\qquad
\partial_{z_{k}}T_{k}(z_{1},\ldots,z_{k})>0.
\label{eq:transport-map-triangular}
\end{equation}
The map pushes the unknown target distribution $\mu$ to a simple reference distribution $\eta$, usually a standard Gaussian. The triangular and monotone structure corresponds to the Knothe--Rosenblatt rearrangement and makes the Jacobian determinant tractable. Under the pullback convention $\mathbf{T}_{\#}\mu=\eta$, the density of $\mu$ can be evaluated by
\begin{equation}
\log p_{\mu}(\mathbf{z})
=
\log p_{\eta}\bigl(\mathbf{T}(\mathbf{z})\bigr)
+
\log \lvert \det\nabla\mathbf{T}(\mathbf{z})\rvert
=
\log p_{\eta}\bigl(\mathbf{T}(\mathbf{z})\bigr)
+
\sum_{k=1}^{d}\log \partial_{z_{k}}T_{k}(z_{1},\ldots,z_{k}).
\label{eq:transport-map-change-of-variables}
\end{equation}
Thus, once the map is learned from samples, density evaluation is reduced to evaluating a reference Gaussian density and a triangular Jacobian determinant. This is the main reason transport maps are useful for estimating EI and synergy in continuous systems, where repeated estimates of intervention-induced joint and marginal densities are required \cite{baptista2024representation}.

In the experiments, we use a conservative affine triangular approximation to this general framework. Given samples $\{\mathbf{z}^{(m)}\}_{m=1}^{M}$, we estimate the empirical mean $\widehat{\boldsymbol{\mu}}$ and the Cholesky factor $\widehat{\mathbf{L}}$ of the covariance matrix, and define
\begin{equation}
\mathbf{T}_{\mathrm{aff}}(\mathbf{z})
=
\widehat{\mathbf{L}}^{-1}(\mathbf{z}-\widehat{\boldsymbol{\mu}}),
\qquad
\log\widehat{p}_{\mathrm{tm}}(\mathbf{z})
=
-\frac{1}{2}
\bigl[d\log(2\pi)+\lVert \mathbf{T}_{\mathrm{aff}}(\mathbf{z})\rVert_{2}^{2}\bigr]
-
\log\lvert\det\widehat{\mathbf{L}}\rvert .
\label{eq:affine-transport-density}
\end{equation}
This estimator should therefore be understood as an affine Gaussian transport-map density model. To capture simple nonlinear source geometry while keeping the estimator stable, the source variables are lifted before density estimation. For one source $x$, we use $(x,x^{2},x^{3})$; for two sources $(x,y)$, we use $(x,y,xy,x^{2},y^{2})$. The cross term $xy$ is included because it is necessary for detecting joint nonlinear mechanisms that cannot be reduced to separable single-source effects.

Given source samples $\mathbf{x}^{(m)}$ and target samples $\mathbf{y}^{(m)}$, mutual information is estimated from the learned transport-map densities by
\begin{equation}
\widehat{I}^{\mathrm{tm}}(X;Y)
=
\frac{1}{M}
\sum_{m=1}^{M}
\bigl[
\log\widehat{p}_{\mathrm{tm}}(\mathbf{x}^{(m)},\mathbf{y}^{(m)})
-
\log\widehat{p}_{\mathrm{tm}}(\mathbf{x}^{(m)})
-
\log\widehat{p}_{\mathrm{tm}}(\mathbf{y}^{(m)})
\bigr]
-
\widehat{b}_{M,d},
\label{eq:transport-map-mi-estimator}
\end{equation}
where $\widehat{b}_{M,d}$ is a finite-sample correction for the affine Gaussian covariance estimate. For two sources, the corresponding synergy estimator is
\begin{equation}
\widehat{\mathrm{Syn}}^{\mathrm{tm}}(X_{1},X_{2};Y)
=
\widehat{I}^{\mathrm{tm}}\bigl(\phi_{12}(X_{1},X_{2});Y\bigr)
-
\widehat{I}^{\mathrm{tm}}\bigl(\phi_{1}(X_{1});Y\bigr)
-
\widehat{I}^{\mathrm{tm}}\bigl(\phi_{2}(X_{2});Y\bigr).
\label{eq:transport-map-two-source-synergy}
\end{equation}
Here $\phi_{1}$, $\phi_{2}$, and $\phi_{12}$ denote the corresponding lifted source representations. Positive values indicate information about the target that is available only when the two sources are considered jointly. Values near zero indicate that the target response can already be explained by single-source effects, while small negative values are treated as finite-sample or density-model mismatch artifacts.

In the controlled continuous experiment, the intervention variables satisfy $Q_{2}^{n},Q_{3}^{n}\sim\operatorname{Uniform}[-L/2,L/2]$, and the target is generated by
\begin{equation}
Q_{1}^{n+1}
=
\alpha\sin(Q_{2}^{n}Q_{3}^{n})
+
(1-\alpha)Q_{2}^{n}
+
\epsilon,
\qquad
\alpha\in[0,1].
\label{eq:alpha-sweep-target}
\end{equation}
When $\alpha=0$, the mechanism is dominated by a single-source term, so the estimated synergy should be close to zero. As $\alpha$ increases, the joint nonlinear term becomes dominant, and the estimated synergy ratio is expected to increase. Changing $L$ further changes the intervention support and therefore tests whether the estimator can distinguish separable source effects from genuinely joint nonlinear source effects. In this sense, the transport-map estimator provides a distribution-level readout of the same structural question addressed by the Jacobian--Covariance formula: whether a target response is explainable by independent source directions or requires a joint source direction.

\section{Details of the Air Pollution Experiment}
\label{sec:air-pollution-experiment-details}

\subsection{Dataset}
\label{subsec:knowairv2-scope}

This appendix provides additional details on the real-world air pollution data used in Section~\ref{sec:applications}. The experiment is based on KnowAir-V2, a large-scale benchmark dataset for air quality forecasting introduced together with PCDCNet, a surrogate model designed to incorporate physical--chemical dynamics and constraints into deep learning-based air quality forecasting \cite{wang2025knowair}. KnowAir-V2 was constructed to support the development and evaluation of data-driven surrogate models that can jointly use pollutant observations, meteorological conditions, and emission-related variables. The dataset has been preprocessed, including missing-value imputation, and is intended to provide a standardized benchmark for interpretable and physically consistent air quality forecasting.

KnowAir-V2 covers the period from 2016 to 2023 and contains $70{,}128$ hourly observations. The full benchmark includes $355$ monitoring stations located in two representative Chinese regions: the Beijing--Tianjin--Hebei and Surrounding Areas (BTHSA) and the Yangtze River Delta (YRD). These two regions are characterized by dense monitoring networks, strong anthropogenic emissions, and complex interactions among pollutant transport, photochemical formation, meteorological modulation, and regional accumulation. The original benchmark uses the years 2016--2019 for training, 2020--2021 for validation and hyperparameter tuning, and 2022--2023 for testing \cite{wang2025pcdcnet}. In the present study, we use the preprocessed hourly records and extract the Hangzhou station network from the YRD portion of the dataset for the continuous effective information analysis.

The variables in KnowAir-V2 can be grouped into three categories: air pollutant observations, meteorological factors, and anthropogenic emission variables. The air pollutant variables include fine particulate matter, $\mathrm{PM}_{2.5}$, and ground-level ozone, O$_3$, which are also the primary forecasting targets in the original benchmark. The meteorological variables describe atmospheric conditions relevant to dispersion, deposition, vertical mixing, and photochemical reactions. They include 2-meter air temperature (\texttt{t2m}), 2-meter dew point temperature (\texttt{d2m}), total precipitation (\texttt{tp}), surface pressure (\texttt{sp}), boundary-layer height (\texttt{blh}), mean surface downward shortwave radiation flux (\texttt{swr}), and the 100-meter wind components (\texttt{u100}, \texttt{v100}). The emission variables include emissions of fine particulate matter, coarse particulate matter, nitrogen oxides, volatile organic compounds, ammonia, and sulfur dioxide. These emission variables provide information about anthropogenic sources and chemical precursors that influence the formation, transport, and removal of $\mathrm{PM}_{2.5}$ and O$_3$.

\begin{table}[htbp]
\centering
\caption{Summary of variable groups in KnowAir-V2 used for the air pollution experiment.}
\label{tab:knowairv2-variable-summary}
\begin{tabular}{lll}
\toprule
Variable group & Variables & Role in the experiment \\
\midrule
Air pollutants & $\mathrm{PM}_{2.5}$, O$_3$ & Observed pollutant states and forecasting targets \\
Meteorology & \texttt{t2m}, \texttt{d2m}, \texttt{tp}, \texttt{sp}, \texttt{blh}, \texttt{swr}, \texttt{u100}, \texttt{v100} & Exogenous atmospheric context \\
Emissions & $\mathrm{PM}_{2.5}$, $PM_{10}$, NOx, VOCs, NH$_3$, SO$_2$ emissions & Anthropogenic source and precursor information \\
\bottomrule
\end{tabular}
\end{table}

The benchmark integrates multiple data sources. Hourly pollutant observations are obtained from ground-based air quality monitoring stations. Meteorological variables are derived from ERA5 reanalysis data, which provides spatially and temporally continuous atmospheric fields. Emission variables are obtained from the Multi-resolution Emission Inventory for China (MEIC). Since emission inventories are usually provided at coarser temporal resolution, the original benchmark aligns them with the hourly pollutant and meteorological records through temporal downscaling procedures \cite{wang2025pcdcnet}. In the original PCDCNet benchmark, the station graph is constructed by connecting monitoring stations within a $200\,\mathrm{km}$ geodesic-distance threshold, so that spatial edges can represent potential transport pathways among nearby stations. In the present analysis, however, the graph visualized in Section~\ref{sec:applications} is not directly taken from this predefined spatial adjacency graph. Instead, station coordinates are used to place nodes at their geographic locations, while the displayed directed edges are inferred from the estimated $EI^{\mathrm{tm}}$ and $Syn^{\mathrm{tm}}$ quantities.

For the application study, we focus on the Hangzhou station network extracted from KnowAir-V2. The input at each time step consists of the current multivariate state of the selected stations, including pollutant and meteorological information. The prediction target is the future O$_3$ state after a $12h$ forecast horizon. This setting is designed to examine whether the continuous PEID framework can recover interpretable station-level information pathways in a realistic air quality forecasting task. In particular, we compare three types of source--target relations: pairwise $\mathrm{O}_3 \to \mathrm{O}_3$ effects, pairwise $\mathrm{PM}_{2.5} \to \mathrm{O}_3$ effects, and bivariate synergistic effects of $\mathrm{O}_3+\mathrm{PM}_{2.5} \to \mathrm{O}_3$ from the same source station to future O$_3$ at target stations.

The purpose of this subset construction is not to identify the station with the highest pollutant concentration, but to quantify how informative each station's current state is for future ozone dynamics across the network. The $\mathrm{O}_3 \to \mathrm{O}_3$ relation mainly captures persistence, transport, and delayed ozone propagation. The $\mathrm{PM}_{2.5} \to \mathrm{O}_3$ relation captures the predictive contribution of particulate matter as a proxy for pollution accumulation, precursor-related conditions, and meteorological background. The bivariate $\mathrm{O}_3+\mathrm{PM}_{2.5} \to \mathrm{O}_3$ relation further tests whether the joint state of ozone and fine particulate matter provides additional non-additive information beyond their individual contributions.

The Hangzhou data are used to train a multistation predictive model, denoted by \texttt{JointStationMLP}, which approximates the local transition from the current multivariate station state to the future pollutant state. After the predictive model is obtained, the continuous PEID procedure is applied to the learned transition model rather than directly to raw temporal correlations. Specifically, the transport-map-based density estimation procedure is used to estimate the mutual-information terms required for continuous $EI^{\mathrm{tm}}$ and $Syn^{\mathrm{tm}}$. This design separates the predictive modeling step from the information-decomposition step: the former learns a smooth multivariate transition function from real air quality data, while the latter evaluates the causal-informational structure of this learned transition under controlled source perturbations.

Because the resulting edges summarize information transfer under the learned transition model, they should be interpreted as station-level causal-informational summaries rather than as direct physical plume trajectories. Their interpretation must therefore be combined with station type, geographic location, local terrain, traffic and industrial activity, meteorological transport, and the photochemical coupling among O$_3$, NOx, VOCs, and $\mathrm{PM}_{2.5}$.

\subsection{Construction and Training Details of the MLP}
\label{subsec:jointstationmlp-training-details}

This subsection supplements the air pollution experiment in Section~\ref{sec:applications} by describing how the predictor used for the transport-map causal graph is constructed and trained. The main text reports only the resulting transport-map causal graph and its scientific interpretation. Here, the purpose is not to propose a new state-of-the-art air quality forecasting model, but to obtain a sufficiently accurate, differentiable, and jointly trained multistation transition function on which $EI^{\mathrm{tm}}$ and $Syn^{\mathrm{tm}}$ can subsequently be estimated.

The representative experiment uses the Hangzhou $12h$ forecasting setting. The raw data are kept at the hourly resolution, and the sample mode is \texttt{one\_step}: each sample uses only the simultaneous multistation snapshot at the current time as input, without explicitly unfolding a historical temporal window. The cached Hangzhou run contains $N=12$ monitoring stations, and each station has $10$ input variables. Therefore, one input sample has shape $12 \times 10$, and the flattened effective input dimension is $120$. The input variables are ordered as two pollutant variables, $\mathrm{O}_3$ and $\mathrm{PM}_{2.5}$, followed by eight meteorological variables, \texttt{t2m}, \texttt{d2m}, \texttt{sp}, \texttt{tp}, \texttt{blh}, \texttt{msdwswrf}, \texttt{u100}, and \texttt{v100}.

The target variables include only $\mathrm{O}_3$ and $\mathrm{PM}_{2.5}$, and they are predicted jointly for all stations. For a forecasting horizon $h=12$, each training sample can be written as
\begin{equation}
\mathbf{x}_t = \bigl(\mathbf{s}_{1,t},\ldots,\mathbf{s}_{N,t}\bigr),
\qquad
\mathbf{y}_{t+h}
=
\bigl(
\mathrm{O}_{3,1,t+h},\mathrm{PM}_{2.5,1,t+h},
\ldots,
\mathrm{O}_{3,N,t+h},\mathrm{PM}_{2.5,N,t+h}
\bigr),
\label{eq:jointstationmlp-sample}
\end{equation}
where $N=12$ and $\mathbf{s}_{i,t}$ is the $10$-dimensional state vector of station $i$ at the current time. Thus, the model output dimension is $24$. This construction ensures that the predictor is not a collection of independent station-wise models, but a single mapping from the joint current state of all stations to the joint future pollutant state of all stations.

The data are split chronologically. Samples whose future target time is no later than \texttt{2021-12-31 23:00:00} are assigned to the training set, samples from 2022 are assigned to the validation set, and samples from 2023 are assigned to the test set. In the current Hangzhou $12h$ run, this produces $52{,}596$ training samples, $8{,}760$ validation samples, and $8{,}760$ test samples. All input and target variables are standardized using statistics computed only from the training period. For each variable, the mean and standard deviation are estimated over all training times and all stations, and the same statistics are reused for validation, testing, and inverse standardization.

The air pollution experiment uses the \texttt{resmlp} variant of \texttt{JointStationMLP}. The input is first flattened into a vector of length $120$ and then passed through a shared trunk. Each forecasting horizon uses an independent linear output head. Since the present experiment considers only the $12h$ horizon, there is a single $12h$ head.

The architecture is composed of an input projection \texttt{Linear(120, 96)} followed by a SiLU activation, three residual MLP blocks, a final \texttt{LayerNorm(96)}, and an output head \texttt{Linear(96, 24)}. Each residual block consists of
\[
\texttt{LayerNorm(96)}
\;\to\;
\texttt{Linear(96, 96)}
\;\to\;
\texttt{SiLU}
\;\to\;
\texttt{Dropout(0.05)}
\;\to\;
\texttt{Linear(96, 96)}
\;\to\;
\texttt{Dropout(0.05)},
\]
and the resulting block output is added back to the block input. In functional notation, the model can be summarized as
\begin{equation}
\widehat{\mathbf{y}}_{t+h}
=
g_h\bigl(f_\theta(\mathbf{x}_t)\bigr),
\label{eq:jointstationmlp-function}
\end{equation}
where $f_\theta$ denotes the shared ResMLP trunk and $g_h$ denotes the horizon-specific linear head. A joint MLP is used instead of separate station-wise predictors because the subsequent causal graph analysis compares cross-station effects such as $\mathrm{O}_3 \to \mathrm{O}_3$, $\mathrm{PM}_{2.5} \to \mathrm{O}_3$, and $\mathrm{O}_3+\mathrm{PM}_{2.5} \to \mathrm{O}_3$. If each station were modeled independently, cross-station interactions in the input would be structurally weakened by the model design itself.

The model is trained with the Adam optimizer. In the current Hangzhou $12h$ configuration, the learning rate is $5\times 10^{-4}$, the weight decay is $10^{-5}$, and the batch size is $64$. Training samples are randomly shuffled within each epoch and updated by mini-batch optimization. The loss function is the sum of mean squared errors over the forecasting horizons:
\begin{equation}
\mathcal{L}(\theta)
=
\sum_{h\in\mathcal{H}}
\frac{1}{B}
\sum_{b=1}^{B}
\left\|
\widehat{\mathbf{y}}_{b,t+h}
-
\mathbf{y}_{b,t+h}
\right\|_2^2 .
\label{eq:jointstationmlp-loss}
\end{equation}
Since the current experiment fixes a single horizon, $\mathcal{H}=\{12\}$, Eq.~\eqref{eq:jointstationmlp-loss} reduces to the multi-output MSE for the $12h$ target.

The training process allows at most $60$ epochs. Model selection is based on validation MSE, and early stopping is triggered if the best validation loss is not updated for $8$ consecutive epochs. In the cached Hangzhou $12h$ run, the selected parameters correspond to epoch $6$, with a best validation loss of $0.3869$, while the training history records up to epoch $14$.

At evaluation time, model outputs are first transformed back from the standardized scale to the original physical scale. RMSE, MAE, and the global flattened Pearson correlation coefficient are then computed. The baseline is a persistence model: it directly uses the current input snapshot of $\mathrm{O}_3$ and $\mathrm{PM}_{2.5}$ as the prediction for all future horizons. This baseline uses the same test set and the same inverse-standardization procedure as \texttt{JointStationMLP}, and therefore provides a minimal threshold for testing whether the learned model contains information beyond simple persistence.

In the current Hangzhou $12h$ test set, the persistence baseline obtains an overall $\mathrm{RMSE}=50.565$, $\mathrm{MAE}=33.119$, and correlation coefficient $0.203$. By contrast, \texttt{JointStationMLP} obtains an overall $\mathrm{RMSE}=26.006$, $\mathrm{MAE}=18.195$, and correlation coefficient $0.764$. This indicates that the subsequent causal graph is not merely explaining noise from a degraded predictor, but is instead applied to a multistation state-transition model that substantially outperforms the persistence baseline.

\end{document}